\documentclass[11pt]{article}
\usepackage[utf8]{inputenc} 
\usepackage[T1]{fontenc}    
\usepackage{wrapfig}

\usepackage{setspace}
\usepackage{cprotect}
\usepackage{amsmath,amssymb,amsthm}
\usepackage[noend]{algorithmic}
\usepackage[ruled,vlined]{algorithm2e}
\usepackage{hyperref}
\usepackage{fullpage}
\usepackage{makeidx}
\usepackage{enumerate}
\usepackage{graphicx,float,psfrag,epsfig}
\usepackage{epstopdf}
\usepackage{color}
\usepackage{enumitem}
\usepackage{caption}
\usepackage{bigints}
\usepackage{mathtools}
\usepackage[mathscr]{euscript}

\usepackage{amsmath}
\usepackage{amssymb}
\usepackage{mathtools}
\usepackage{amsthm}
\usepackage{amsfonts}
\usepackage{soul}
\usepackage{csquotes}
\usepackage{amsmath}
\usepackage{amsthm}
\usepackage{amsfonts}
\usepackage{hyperref}
\usepackage{comment}
\usepackage{graphicx}
\usepackage{xcolor}
\usepackage{parskip}
\usepackage{hyperref}       
\usepackage{url}            
\usepackage{booktabs}       
\usepackage{amsfonts}       
\usepackage{nicefrac}       
\usepackage{microtype}      
\usepackage{xcolor}         
\usepackage{microtype}
\usepackage{graphicx}
\usepackage{subfigure}  
\usepackage{amsfonts}
\usepackage{amsmath}
\usepackage{amssymb}
\usepackage{bbm}
\usepackage{multirow}
\usepackage{mathtools}
\usepackage{relsize}

\DeclareMathOperator{\Rank}{rank}
\DeclareMathOperator{\Span}{span}
\DeclareMathOperator{\Rows}{rows}

\DeclareMathOperator{\Flat}{flat}
\DeclareMathOperator{\softmax}{softmax}

\def\Pp{P_{\Phi}}

\def\P{\mathbb{P}}

\def\R{\mathbb{R}}

\newcommand{\opnorm}[1]{\left\lVert#1\right\rVert_{\textup{op}}}

\def\b0{{0}}

\def\RR{\mathbb{R}}

\def\>{\rangle}

\newcommand{\E}{\mathbb{E}}

\newcommand{\distas}[1]{\mathbin{\overset{#1}{\sim}}}

\newcommand{\bigO}[1]{\mathcal{O}\left(#1\right)}

\def\Span{\textrm{Span}}

\newcommand{\norm}[1]{\left\|#1\right\|}
\newcommand{\subGnorm}[1]{\left\|#1\right\|_{\psi_2}}
\newcommand{\subEnorm}[1]{\left\|#1\right\|_{\psi_1}}

\newcommand{\abs}[1]{\left|#1\right|}

\newcommand{\evmin}[1]{\lambda_{\rm min}\left(#1\right)}

\def\Lip{\mathrm{Lip}}

\def\PP{\mathbb{P}}

\def\min{\mathop{\rm min}\nolimits}
\def\max{\mathop{\rm max}\nolimits}

\numberwithin{equation}{section}

\newtheoremstyle{myexample} 
    {\topsep}                    
    {\topsep}                    
    {\rm }                   
    {}                           
    {\bf }                   
    {.}                          
    {.5em}                       
    {}  

\newtheoremstyle{myremark} 
    {\topsep}                    
    {\topsep}                    
    {\rm}                        
    {}                           
    {\bf}                        
    {.}                          
    {.5em}                       
    {}  

\newtheorem{claim}{Claim}[section]
\newtheorem{lemma}[claim]{Lemma}

\newtheorem{assumption}{Assumption}

\newtheorem{theorem}{Theorem}

\theoremstyle{myremark}
\newtheorem{remark}{Remark}[section]
\theoremstyle{myremark}

\theoremstyle{myexample}

\author{Simone Bombari\thanks{Institute of Science and Technology Austria (ISTA). Emails: \texttt{\{simone.bombari, marco.mondelli\}@ist.ac.at}.}\;,
\;\;Marco Mondelli\footnotemark[1].}

\title{Towards Understanding the Word Sensitivity of Attention Layers: \\
A Study via Random Features}

\begin{document}

\maketitle

\begin{abstract}
\vspace{0.2cm}

Understanding the reasons behind the exceptional success of transformers requires a better analysis of why attention layers are suitable for NLP tasks. In particular, such tasks require predictive models to capture contextual meaning which often depends on one or few words, even if the sentence is long. Our work studies this key property, dubbed \emph{word sensitivity} (WS), in the prototypical setting of random features. We show that attention layers enjoy high WS, namely, there exists a vector in the space of embeddings that largely perturbs the random attention features map. The argument critically exploits the role of the $\softmax$ in the attention layer, highlighting its benefit compared to other activations (e.g., ReLU). In contrast, the WS of standard random features is of order $1/\sqrt{n}$, $n$ being the number of words in the textual sample, and thus it decays with the length of the context. We then translate these results on the word sensitivity into generalization bounds: due to their low WS, random features provably cannot learn to distinguish between two sentences that differ only in a single word; in contrast, due to their high WS, random attention features have higher generalization capabilities. We validate our theoretical results with experimental evidence over the BERT-Base word embeddings of the imdb review dataset.

\end{abstract}

\section{Introduction}

Deep learning theory has provided a quantitative description of phenomena routinely occuring in state-of-the-art models, such as double-descent \cite{Nakkiran2020Deep,mei2022generalization}, benign overfitting \cite{bartlett20benign, belkin2021}, and feature learning \cite{ba2022highdimensional, damian2022neural}. However, most existing works focus on architectures given by the composition of matrix multiplications and 
non-linearities, which model e.g.\ fully connected and convolutional layers. In contrast, the recent impressive results achieved by large language models \cite{gpt3, bubeck2023sparks} are largely attributed to the introduction of transformers \cite{vaswani17}, which are in turn based on attention layers \cite{bahdanau15translation, kim2017structured}. 
Hence, isolating the unique features of the attention mechanism stands out as a critical challenge to understand the success of transformers, thus paving the way to the principled design of large language models.

Recent work tackling this problem 
characterizes the sample complexities required by simplified attention models \cite{edelman22inductive, fu2023single}.
However, learning is limited to a specific set of targets, such as sparse functions \cite{edelman22inductive} 
or functions of the correlations between the first query token and key tokens \cite{fu2023single}. 
This paper 
takes a different perspective and starts from the simple empirical observation that \emph{one or few words can change the meaning of a sentence}. Think, for example, to the pair of sentences ``I love her much'' and ``I love her smile'', where replacing a single word alters the meaning of the text. This is captured by the BERT-Base model \cite{bert}, which shows a different attention score pattern over these two sentences (see Figure \ref{fig:att_scores}, left).
Another example can be found in the table on the right of Figure \ref{fig:att_scores}, where changing one word in the prompt would require a well-behaved model (in this case, Llama2-7b \cite{llama2}) to modify its output. In general, language models need to have a high \emph{word sensitivity} to capture semantic changes when just a single word is modified in the context, which motivates the following question: 

\vspace{0.2em}
\begin{center}
    \textit{Do attention layers have a larger word sensitivity than fully connected architectures?}
\end{center}
\vspace{0.2em}

\begin{figure*}[t]
\begin{minipage}{0.42\textwidth}
    \centering
    \includegraphics[width=\textwidth]{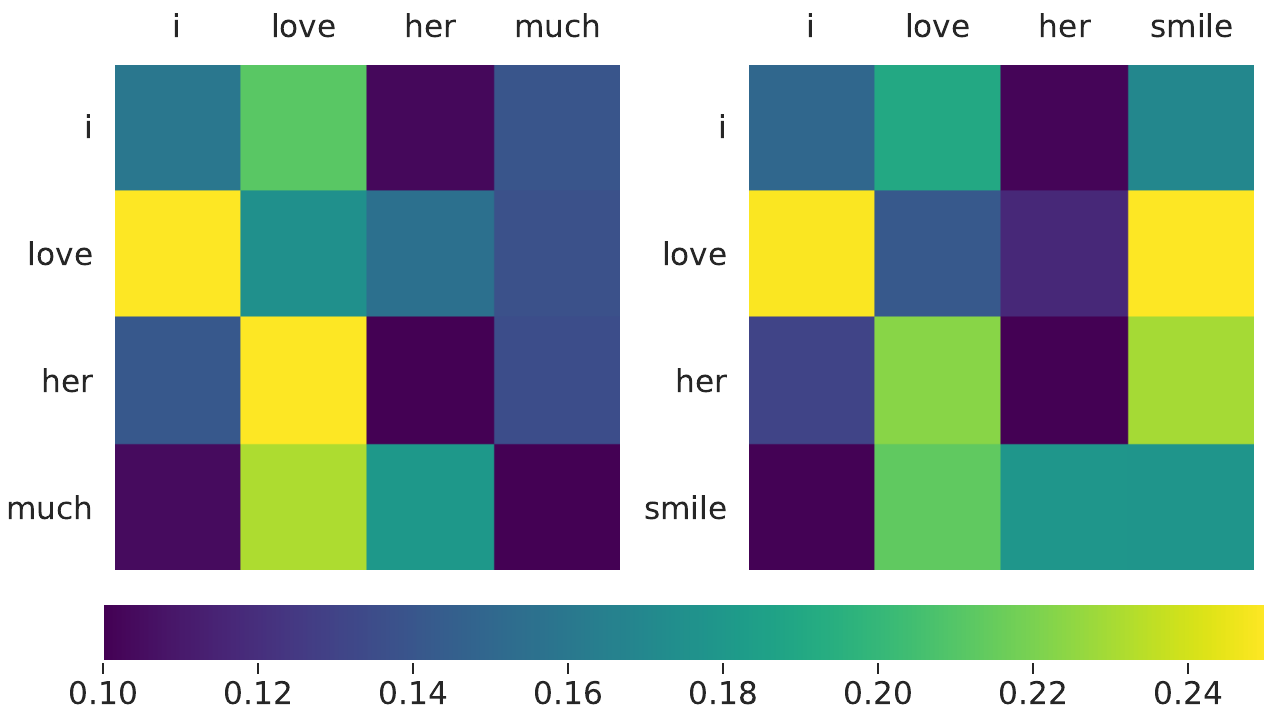}
\end{minipage}
\hspace{\fill}
\begin{minipage}{0.55\textwidth}
    \begin{table}[H]
        \centering
        \begin{tabular}{|p{7cm}|p{1cm}|}
            \hline
            \small{Prompt} & \small{Output} \\
            \hline
            \footnotesize{Reply with "Yes" if the review I will provide you is \textcolor{blue}{positive}, and "No" otherwise.}
            
            \footnotesize{Review: Sorry, gave it a 1, which is the rating I give to movies on which I walk out or fall asleep.} & \footnotesize{No} \\
            \hline
            \footnotesize{Reply with "Yes" if the review I will provide you is \textcolor{blue}{negative}, and "No" otherwise.}
            
            \footnotesize{Review: Sorry, gave it a 1, which is the rating I give to movies on which I walk out or fall asleep.} & \footnotesize{Yes} \\
            \hline
        \end{tabular}
    \end{table}
\end{minipage}
\caption{\emph{Left.} Average attention scores for the word embeddings of the two sentences ``I love her much'' and ``I love her smile''. The embeddings are computed with the BERT-Base model, the scores are averaged over the 12 heads and displayed without the \texttt{[CLS]} token. \emph{Right.} Output of the Llama2-7b-chat model for two prompts differing only in a single word.}
\label{fig:att_scores}
\vspace{-1em}
\end{figure*}

To formalize the problem, we represent the textual data with $X \in \R^{n \times d}$, where the $n$ rows represent the word embeddings in $\R^d$. Then, we define the word sensitivity (WS) in \eqref{eq:sensitivity}, as a measure of how changing a row in $X$ modifies the embedding of a given feature map. We focus our study on the prototypical setting of \emph{random features}, i.e., where the weights of the layers are random. In particular, we consider \emph{(i)} the \emph{random features} (RF) map \cite{rahimi2007random}, defined in \eqref{eq:rf}, and \emph{(ii)} the \emph{random attention features} (RAF) map \cite{fu2023single}, defined in \eqref{eq:raf}. The former models a fully connected architecture, while the latter captures the structure typical of attention layers.

Our contributions can be summarized as follows:

\begin{itemize}
    \item Theorem \ref{thm:rfS} shows that an RF map has \emph{low} WS, specifically of order $1/\sqrt{n}$, where $n$ is the context length. This means that changing a single word has a negligible effect on the output of the map. In fact, to have a significant effect, one needs to change a constant fraction of the words, see Remark \ref{rem:constant}. Furthermore, increasing the depth of the architecture does not help with the word sensitivity, as shown by Theorem \ref{thm:drf}. 

    \item Our main result, Theorem \ref{thm:RAF}, shows that a RAF map has \emph{high} WS, specifically of constant order which does not depend on the length of the context. This means that changing even a single word can have a significant effect on the output of the map, regardless of the context length. 
    The argument critically exploits the role of the $\softmax$ in the attention layer, and numerical simulations show its advantages compared to other activations (e.g., ReLU).

    \item Section \ref{sec:gen} exploits the bounds on the word sensitivity to characterize the generalization error when the data changes meaning after modifying a single word in the context. In particular, we consider generalized linear models trained on RF and RAF embeddings, and we establish whether a fine-tuned or retrained model can learn to distinguish between two samples that differ only in one row. While the answer is provably negative for random features (Theorems \ref{thm:erroronDfinetuning} and \ref{thm:erroronD}), random attention features are capable of generalizing.
\end{itemize}

Most of our technical contributions require no distributional assumptions on the data, and the generality of our findings is confirmed by numerical results on the BERT-Base embeddings of the imdb dataset \cite{imdb}, and on the pre-trained BERT-Base model itself. Our code is publicly available at the GitHub repository \href{https://github.com/simone-bombari/attention-sensitivity}{\texttt{simone-bombari/attention-sensitivity}}.

\section{Related work}\label{sec:rel}

\paragraph{Fully connected layers.}

Several mathematical models have been proposed to understand phenomena occurring for fully connected architectures. A prototypical example is the RF model \cite{rahimi2007random, pennington17nonlinear, louart2018random, mei2022generalization}, which can be thought of as a two-layer neural network with random hidden weights. Its feature learning capabilities have been recently studied in settings where one gradient step on the hidden weights is performed before the final training of the outer layer \cite{ba2022highdimensional, ba2023learning, damian2022neural, moniri2023theory}.
Other popular approaches involve the neural tangent kernel \cite{JacotEtc2018, lee19wide, ghorbani20when, ghorbani2021linearized} and a mean-field analysis \cite{mei18mean, sirignano2020mean2, chizat2018global, rotskoff2018neural, javanmard2020analysis, Shevchenko2021Meanfield}. 
Deep random models have also been considered by \cite{hanin19universal, bosch23aymptotic, schroder2023deterministic}.

\paragraph{Attention layers.}

Attention layers \cite{bahdanau15translation, kim2017structured} and transformer architectures \cite{vaswani17} have attracted significant interest from the theoretical community: 
\cite{Yun2020Are, ben-shaul2023exploring} study their approximation capabilities; \cite{edelman22inductive, trauger2023sequence} provide norm-based generalization bounds; \cite{tian2023scan} analyze the training dynamics; \cite{wu2023on} provide optimization guarantees; \cite{jelassi2022vision, li2023a} focus on computer vision tasks and \cite{oymak23attention} on prompt-tuning. The study of the attention mechanism is approached through the lens of associative memories by \cite{bietti2023birth, cabannes2024scaling}. More closely related to our setting is the recent work by \cite{fu2023single}, which compares the sample complexity of random attention features with that of random features. We highlight that \cite{fu2023single} focus on an attention layer with ReLU activation, while we unveil the critical role of the $\softmax$.



\paragraph{Sensitivity of neural networks.}

We informally use the term sensitivity to express how a perturbation of the input changes the output of the model. Previous work explored various mathematical formulations of this concept (e.g., the input-output Lipschitz constant) in the context of both robustness \cite{weng2018evaluating, bubeck2021a} and generalization \cite{bartlett17spectral}. Sensitivity is generally referred to as an \emph{undesirable} property, motivating research on models that reduce it \cite{miyato2018spectral, prach22orthogonal}. In this work, however, high sensitivity represents a \emph{desirable} attribute, as it reflects the ability of the model to capture the role of individual words in a long context. This is a stronger requirement than having large Lipschitz constant, hence earlier results on the matter \cite{hyunjik21lipschitz} cannot be applied.


\section{Preliminaries}

We consider a sequence of $n$ tokens $\{x_i\}_{i = 1}^n$, with $x_i \in \R^d$ for every $i$, where $d$ denotes the token embedding dimension, and $n$ the context length. These tokens altogether represent the textual sample $X = [x_1, \ldots, x_n] ^\top \in \R^{n \times d}$.
We denote by $\Flat(X)\in \R^D$, with $D=nd$, the flattened (or vectorized) version of $X$. 
Given a vector $x$, $\norm{x}_2$ is its Euclidean norm. Given a matrix $A$, $\norm{A}_2 := \norm{\Flat(A)}_2 \equiv \norm{A}_F$ is its Frobenius norm. We indicate with $e_i$ the $i$-th element of the canonical basis, and denote $[n]:=\{1, \ldots, n\}$. All the complexity notations $\Omega(\cdot)$, $\omega(\cdot)$, $\mathcal{O}(\cdot)$, $o(\cdot)$ and $\Theta(\cdot)$ are understood for sufficiently large context length $n$, token embedding dimension $d$, number of neurons $k$, and number of input samples $N$. We indicate with $C,c>0$ numerical constants, independent of $n, d, k, N$. Throughout the paper, we make the following assumption, which is easily achieved by pre-processing the raw data. 
\begin{assumption}[Normalization of token embedding]\label{ass:d}
    For every token $x_i$, we assume $\norm{x_i}_2 = \sqrt d$.
\end{assumption}


\paragraph{Random Features (RF).}
A fully connected layer with random weights is commonly referred to 
as a \emph{random features} map \cite{rahimi2007random}. The map 
acts from a vector of covariates to a feature space $\R^k$, where $k$ denotes the number of neurons. Thus, we flatten the context $\Flat(X) \in \R^D$ before feeding it in the layer, and the RF map $\varphi_{\textup{RF}}: \R^{n \times d} \to \R^{k}$ takes the form
\begin{equation}\label{eq:rf}
    \varphi_{\textup{RF}}(X) = \phi (V \Flat(X)),
\end{equation}
where $\phi: \R \to \R$ is a non-linearity applied component-wise and $V \in \R^{k \times D}$ is the random features matrix, with $V_{i,j} \distas{}_{\rm i.i.d.}\mathcal{N}(0, 1 / D)$. This scaling of the variance of $V$ ensures that the entries of $V \Flat(X)$ have unit variance, as $\norm{\Flat X}_2 = \sqrt{D}$ by Assumption \ref{ass:d}.
We later consider a similar model with several random layers, referred to as \emph{deep random features} (DRF) model and recently considered by \cite{bosch23aymptotic,schroder2023deterministic}.

\paragraph{Random Attention Features (RAF).}
We consider a single-head sequence-to-sequence self-attention layer without biases $\varphi_{\textup{QKV}}(X)$ \cite{vaswani17}, given by 
\begin{equation}\label{eq:qkv}
    \varphi_{\textup{QKV}}(X) = \softmax \left( \frac{X W_Q^\top W_K X^\top}{\sqrt{d'}} \right) X W_V^\top,
\end{equation}
where the $\softmax$ is applied row-wise and defined as $\softmax (s)_i = e^{s_i} / \sum_j e^{s_j}$; $W_Q, W_K, W_V\in \R^{d' \times d}$ are respectively the queries, keys and values weight matrices.
As in previous related work \cite{fu2023single}, we simplify the above expression with the re-parameterization $W := W_Q^\top W_K \in \R^{d \times d}$, and by removing the values weight matrix. This is for convenience of presentation, and our results can be generalized to the case where queries, keys and values are random independent features, see Remark \ref{rem:rep} at the end of Section \ref{sec:raf}. Thus, we define the \emph{random attention features} layer $\varphi_{\textup{RAF}}: \R^{n \times d} \to \R^{n \times d}$ as
\begin{equation}\label{eq:raf}
    \varphi_{\textup{RAF}}(X) = \softmax \left( \frac{X W X^\top}{\sqrt{d}} \right) X,
\end{equation}
where $W_{i,j} \distas{}_{\rm i.i.d.}\mathcal{N}(0, 1 / d)$. We refer to the argument of the $\softmax$ with the shorthand $S(X) := X W X^\top /\sqrt{d} \in \R^{n \times n}$. The scaling of the variance of $W$ ensures that the entries of $S(X)$ have unit variance. 
Differently from \cite{fu2023single}, we do not consider the biased initialization discussed in \cite{asher23mimetic}, designed to make the diagonal elements of $W$ positive in expectation. As remarked in \cite{asher23mimetic}, this initialization is aimed at replicating the final attention scores of vision transformers, and its utility on language models is less discussed. Furthermore, we remark that our simplified model does not include any of the architectural tweaks introduced to allow transformer models to process longer contexts \cite{bertsch2023unlimiformer, Beltagy2020Longformer}, as it aims to capture the specific properties of a single attention layer.

\paragraph{Word sensitivity (WS).}
The word sensitivity (WS) measures how the embedding of a given mapping $\varphi(X)$ is impacted by a change in a single word. Given a sample $X=[x_1, \ldots, x_n]^\top\in\mathbb R^{n\times d}$, the \emph{perturbed} sample $X^i(\Delta)$ is obtained from $X$ by setting its $i$-th row to $x_i + \Delta$ and keeping the remaining rows the same. Here, $\Delta \in \R^d$ denotes the perturbation of the $i$-th token, and its magnitude does not exceed the scaling of Assumption \ref{ass:d}, i.e., $\norm{\Delta}_2 \leq \sqrt{d}$. Formally, given a mapping $\varphi$, we are interested in 
\vspace{.7em}
\begin{equation}\label{eq:sensitivity}
    \mathcal S_{\varphi}(X) = \sup_{i \in [n], \, \norm{\Delta}_2 \leq \sqrt{d}} \frac{\norm{\varphi(X^i(\Delta)) - \varphi(X)}_2}{\norm{\varphi(X)}_2}.
    \vspace{.7em}
\end{equation}
In words, $\mathcal S_\varphi(X)$ denotes the highest relative change of $\varphi(X)$, upon changing a single token in the input. Our goal is to study how $\mathcal S_\varphi(X)$ behaves for the RF and RAF models, with respect to the context length $n$. Informally, if $\mathcal S_\varphi(X) = o(1)$ for any $X$, the mapping $\varphi$ has \emph{low word sensitivity}. On the contrary, if $\mathcal S_\varphi(X) = \Omega(1)$
, $\varphi$ has \emph{high word 
sensitivity}.

\vspace{.1em}

We remark that, in language models, tokens are elements of a discrete vocabulary. However, working directly in the embedding space (as in definition \eqref{eq:sensitivity}) is common in the theoretical literature \cite{hyunjik21lipschitz}, and critical steps of practical algorithms for language models also take place in the embedding space \cite{ebrahimi2018hotflip, shin2020autoprompt, zou2023universal}.

\vspace{.1em}

The definition of WS recalls similar notions of sensitivity in the context of adversarial robustness \cite{dohmatob2022non, bubeck2021a, wu2021do, bombari2023universal}, as well as the literature that designs adversarial prompting schemes for language models \cite{ebrahimi2018hotflip, guo2021gradient, zou2023universal}.
However, in contrast with the definition in \eqref{eq:sensitivity}, these works consider a \emph{trained} model and, therefore, the results implicitly depend on the training dataset. Our analysis of the WS dissects the impact of the feature map $\varphi$, and it does not involve any training process. The consequences on training (and generalization error) will be considered in Section \ref{sec:gen}. 

\vspace{.1em}

\section{Low WS of random features}\label{sec:rf}

We start by showing that the word sensitivity for the RF map in \eqref{eq:rf} is \emph{low}, i.e., 
$\mathcal S_{{\textup{RF}}} = o(1)$.
\begin{theorem}\label{thm:rfS}
     Let $\varphi_{\textup{RF}}$ be the random features map defined in \eqref{eq:rf}, where $\phi$ is Lipschitz and not identically $0$. Let $X \in \R^{n \times d}$ be a generic input sample s.t.\ Assumption \ref{ass:d} 
     holds, and assume $k = \Omega(D)$. Let $\mathcal S_{\textup{RF}}(X)$ denote the word sensitivity of $\varphi_{\textup{RF}}$ defined in \eqref{eq:sensitivity}. 
     Then, we have
     \vspace{.7em}
     \begin{equation}\label{eq:rfnr}
          \mathcal S_{\textup{RF}}(X) = \bigO{1/\sqrt n} = o(1),
     \vspace{.5em}
     \end{equation}
     with probability at least $1 - \exp(-c D)$ over $V$.
\end{theorem}
\vspace{.3em}
Theorem \ref{thm:rfS} shows that the RF model has a low word sensitivity, vanishing with the length of the context $n$. This means that, regardless of how any word is modified, the RF mapping is \emph{not sensitive} to this modification, when the length of the context is large. 
The proof 
follows from an upper bound on the numerator of \eqref{eq:sensitivity}, due to the Lipschitz continuity of $\phi$, 
and a concentration result on the norm at the denominator. 
The details 
are deferred to Appendix \ref{app:rf}. 
\begin{remark}\label{rem:constant}
    Theorem \ref{thm:rfS} is readily extended to the case where the number of words $m$ that can be modified in the context is $o(n)$. In fact, to achieve a sensitivity of constant order, a constant fraction of rows in $X$ must be changed, i.e. $m = \Theta(n)$. This extension of Theorem \ref{thm:rfS} is also proved in Appendix \ref{app:rf}, and it further illustrates that a fully connected architecture is unable to capture the change of few (namely, $m = o(n)$) semantically relevant words in a long sentence.
\end{remark}

We numerically validate this result in Figure \ref{fig:RF_S} (left), where we estimate the value of $\mathcal S_{\textup{RF}}$ for different token embedding dimensions $d$, as the context length $n$ increases. 
Clearly, $\mathcal S_{\textup{RF}}$ decreases as the context length increases, regardless of the embedding dimension $d$. In fact, the curves corresponding to different values of $d$ (in different colors) basically coincide.



\begin{figure}[!t]
    \centering
    \includegraphics[width=0.7\textwidth]{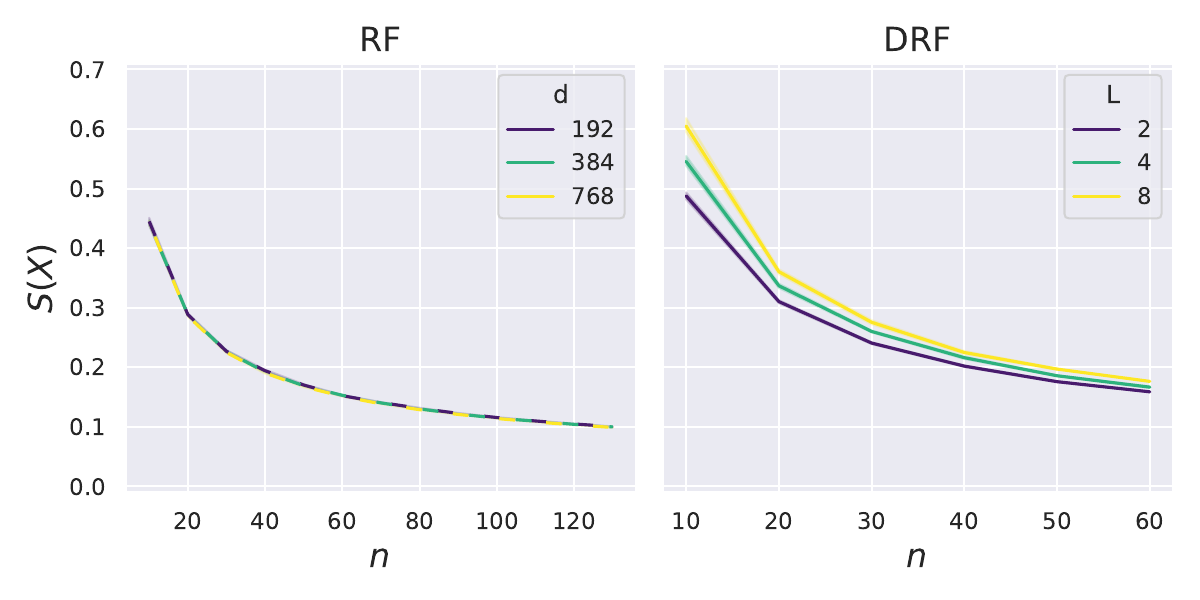}
    \caption{Numerical estimate of the WS for the RF (left) and DRF (right) map, with a ReLU activation function. We estimate $\mathcal S_{\varphi}$ looking for the perturbation $\Delta^* = \arg \sup_{\norm{\Delta}_2 \leq \sqrt{d}} \norm{\varphi(X^1(\Delta)) - \varphi(X)}_2$, where we fix the first token for symmetry. We find $\Delta^*$ by optimizing our objective with constrained gradient ascent. 
    For RF, we consider $d\in\{192, 384, 768\}$, as $n$ increases (taking the first $d$ dimensions of the embeddings of the first $n$ tokens). For DRF, we repeat the experiment for different depths $L\in\{2, 4, 8\}$ and fixed $d = 768$, as $n$ increases. As textual data $X$, we use the BERT-Base token embeddings of samples from the imdb dataset, after a pre-processing to adapt the dimensions and fulfill Assumption \ref{ass:d}. We plot the average over 10 independent trials and the confidence band at 1 standard deviation. In the figure on the left, we intentionally dash the plotted lines to ease the visualization, as they overlap.
    }
    \label{fig:RF_S}
\end{figure}

\paragraph{Deep Random Features (DRF).} 

\vspace{-0.3em}

As an extension of Theorem \ref{thm:rfS}, we show that the \emph{word sensitivity is low} also for \emph{deep} random features. We consider the 
DRF map
\begin{equation}\label{eq:drf}
    \varphi_{\textup{DRF}}(X) := \phi(V_L \phi( V_{L-1} (... V_2 \phi (V_1 \Flat(X))  ... ))),
\end{equation}
where $\phi: \R \to \R$ is the component-wise non-linearity, $k$ is the number of neurons at each layer, 
and $V_l \in \R^{k \times k}$ are the random weights at layer $l$, with $[V_1]_{i,j} \distas{}_{\rm i.i.d.}\mathcal{N}(0, \beta / D)$ and $[V_l]_{i,j} \distas{}_{\rm i.i.d.}\mathcal{N}(0, \beta / k)$ for $l>1$.
We set $\beta$ according to He's initialization \cite{He2015}, which ensures 
\begin{equation}\label{eq:He}
\E_{\rho \sim \mathcal N(0, \beta)} \left[ \phi^2( \rho ) \right] = 1,
\end{equation}
as done in \cite{hanin18which} to avoid the problem of vanishing/exploding gradients in the analysis of a deep 
network.

\begin{theorem}\label{thm:drf}
Let $\varphi_{\textup{DRF}}$ be the deep random feature map defined in \eqref{eq:drf}, where $\phi$ is Lipschitz and $\beta$ is chosen s.t.\ \eqref{eq:He} holds. 
Let $X \in \R^{n \times d}$ be a generic input sample s.t.\ Assumption \ref{ass:d} holds, and assume $k = \Theta(D)$, $L = o (\log k)$. Let $\mathcal S_{\textup{DRF}}(X)$ denote the word sensitivity of $\varphi_{\textup{DRF}}$ defined in \eqref{eq:sensitivity}. Then, we have
\vspace{0.8em}
\begin{equation}
    \mathcal S_{\textup{DRF}} = \bigO{\frac{e^{CL}}{\sqrt n}},
\vspace{.7em}
\end{equation}
with probability at least $1 - \exp(-c \log^2 d)$ over $\{ V_l \}_{l = 1}^L$. 
\end{theorem}
\vspace{.3em}
Theorem \ref{thm:drf} shows that, as in the shallow case, the word sensitivity decreases with the length of the context $n$. The proof requires showing that $\norm{\varphi_{\textup{DRF}}(X)}_2$ concentrates to $\sqrt k$, which is achieved via \eqref{eq:He}. 
The strategy to bound the term $\norm{\varphi_{\textup{DRF}}(X^i(\Delta)) - \varphi_{\textup{DRF}}(X)}_2$ is similar to that of the shallow case. The
details are deferred to 
Appendix \ref{app:drf}.

The exponential dependence on $L$ comes from our worst-case analysis of the Lipschitz constant of the model, and it does not fully exploit the independence between the $V_i$'s. Thus, we expect the actual dependence of the WS on $L$ to be milder. This is confirmed by the numerical results of Figure \ref{fig:RF_S} (right), where we numerically estimate 
$\mathcal S_{\textup{DRF}}$ for different depths $L$ as the context length $n$ increases. For all values of $L$, 
the word sensitivity quickly 
decreases with $n$.


\section{High WS of random attention features}\label{sec:raf}

In contrast with random features, we show that the \emph{word sensitivity is high} for the RAF map in \eqref{eq:raf}, i.e.,  $\mathcal S_{{\textup{RAF}}} = \Omega(1)$.
\begin{theorem}\label{thm:RAF}
Let $\varphi_{\textup{RAF}}$ be the random attention features map defined in \eqref{eq:raf}. Let $X \in \R^{n \times d}$ be a generic input sample s.t.\ Assumption \ref{ass:d} holds, and assume $ d / \log^4 d = \Omega(n)$. Let $\mathcal S_{\textup{RAF}}(X)$ denote the word sensitivity of $\varphi_{\textup{RAF}}$ defined in \eqref{eq:sensitivity}. Then, we have
\vspace{0.8em}
\begin{equation}
    \mathcal S_{{\textup{RAF}}}(X) = \Omega(1),
\vspace{0.6em}
\end{equation}
with probability at least $1 - \exp(-c \log^2 d)$ over $W$.
\end{theorem}
\vspace{0.2em}
Theorem \ref{thm:RAF} shows that the RAF model has a high word sensitivity, regardless of the length of the context $n$, as long as $d / \log^4 d = \Omega(n)$. This requires a number of tokens $n$ that grows slower than the embedding dimension $d$. Models such as BERT-Base or BERT-Large have an embedding dimension of 768 and 1024 \cite{bert}, which allows our results to hold for fairly large context lengths.
In fact, we prove a stronger statement: 
for \emph{any} index $i \in [n]$, there exists a perturbation $\Delta^*$ (possibly dependent on $i$) s.t.\ the RAF map changes significantly when evaluated in $X^i(\Delta^*)$. 
The proof of Theorem \ref{thm:RAF} is deferred to Appendix \ref{app:raf}, and a sketch follows.


\paragraph{Proof sketch.} The argument is not constructive, and the difficulty in finding a closed-form solution for the perturbation $\Delta^*$ is due to the lack of assumptions on the sample $X$, which is entirely generic. We follow the steps below. 

\emph{\ul{Step 1}: Find a direction $\delta^*$ aligned with many words $x_j$'s.} Using the probabilistic method, we prove the existence of a vector $\delta^* \in \R^d$, with $\norm{\delta^*}_2 \leq \sqrt d$, s.t.\ its inner product with the tokens embeddings $x_j$'s is large for a constant fraction of the words in the context. In particular, Lemma \ref{lemma:dstar} shows that $\left( x_j^\top \delta^* \right)^2 = \Omega ( d^2 / n )$ for at least $\Omega(n)$ indices $j \in [n]$. 

\vspace{1em}
\emph{\ul{Step 2}: Exhibit two directions $\Delta^*_1$ and $\Delta^*_2$ both aligned with many words in the feature space $\{W^\top x_j\}_{j=1}^n$.} Exploiting the properties of the Gaussian attention features $W$, we deduce the existence of two vectors $\Delta^*_1$ and $\Delta^*_2$ that \emph{(i)} are far from each other, and \emph{(ii)} ensure a constant fraction of the entries of $X W \Delta^*_k /\sqrt d$, $k\in \{1, 2\}$, to be large. In particular, by exploiting our assumption $d /\log^4 d = \Omega(n)$, Lemmas \ref{lemma:dstar12} and \ref{lemma:Dstar} show that $[X W \Delta^*_k / \sqrt d ]_j = \Omega \left( \log^2 d \right)$ for at least $\Omega(n)$ indices $j \in [n]$, with $\norm{\Delta^*_1 - \Delta^*_2}_2 = \Omega(\sqrt d)$. 

\vspace{1em}
\emph{\ul{Step 3}: Show that the attention concentrates towards the perturbed word.} 
Recall that $s(X) := \softmax( X W X^\top / \sqrt{d})$ denotes the attention scores matrix. Then, Lemma \ref{lemma:scoresstar} proves that the attention scores $[s(X^i(\Delta^*_{k}))]_{j:}^\top$ are well approximated by the canonical basis vector $e_i$, for $k\in \{1, 2\}$ and an $\Omega(n)$ number of rows $j \in [n]$. This intuitively means that a constant fraction of tokens $x_j$ \emph{moves all their attention} towards the $i$-th modified token $x_i + \Delta^*_{k}$. This step critically exploits the $\softmax$ function in the RAF map: if one entry in its argument is $\Omega \left( \log^2 d \right)$, then the attention scores concentrate as described above. 

\vspace{1em}
\emph{\ul{Step 4}: Conclude with at least one perturbation between $\Delta^*_1$ and $\Delta^*_2$.} Finally, exploiting the fact that $\norm{\Delta^*_1 - \Delta^*_2}_2 = \Omega(\sqrt d)$, Lemma \ref{lemma:numeratoromega} proves that at least one of them gives a $\Delta^*$ s.t.\ $\norm{\varphi_{\textup{RAF}}(X) - \varphi_{\textup{RAF}}(X^i(\Delta^*))}_F = \Omega(\sqrt {dn})$. This, together with the upper bound $\norm{\varphi_{\textup{RAF}}(X)}_F = \mathcal O (\sqrt{dn})$, concludes the argument. \qed 

\vspace{1em}

\begin{figure*}[!t]
    \centering
    \includegraphics[width=\textwidth]{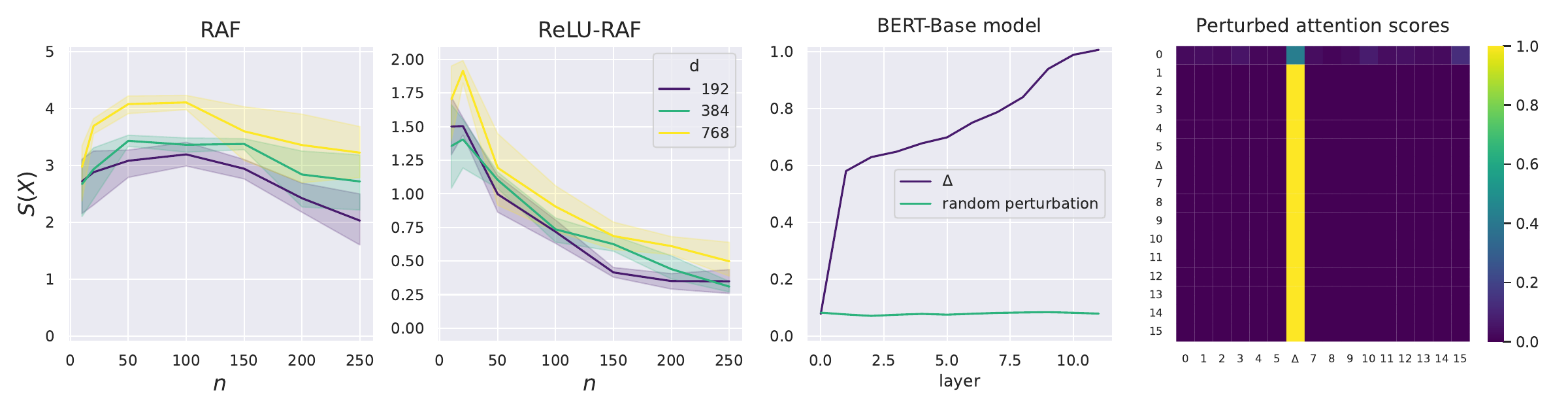}
    \vspace{0.3em}
    \caption{Numerical estimate of the WS for the RAF (first plot) and ReLU-RAF (second plot) map. The ReLU-RAF map is defined as the RAF one, but the $\softmax$ is replaced with a ReLU activation over the entries of $S(X)$, followed by a re-normalization to ensure that the attention scores sum up (on every row) to 1, as in the $\softmax$ case. We consider $d={192, 384, 768}$, as $n$ increases. The rest of the setup is equivalent to the one described in Figure \ref{fig:RF_S}. In the third plot, we present the relative change of the pre-trained BERT-Base model layer embeddings, evaluated on the abstract of this paper ($\sim$ 200 tokens), when the 42nd token is modified in embedding space with a vector $\Delta$. The $0$-th layer represents the input itself and, as a comparison, we report the results when the perturbation is chosen to be Gaussian noise with the same norm. In the fourth plot, we present the attention scores in the first head of the first layer of BERT-Base model, evaluated on the title of this paper, when the 6th token is modified in embedding space with a vector $\Delta$. In the third and fourth plots, $\Delta$ is chosen to be a perturbation that \emph{attracts} all the attention on the perturbed key token, which follows the proof idea of Theorem \ref{thm:RAF}.}
    \vspace{0.3em}
    \label{fig:RAF_S}
\end{figure*}

In Figure \ref{fig:RAF_S} (first plot), we estimate the value of $\mathcal S_{\textup{RAF}}$ for different token embedding dimensions $d$, as the context length $n$ increases. In contrast with the random features map, even for large values of $n$, the WS remains larger than $1$.
In the same figure, in the second plot, we repeat the experiment for the ReLU-RAF map, which replaces the $\softmax$ with a ReLU activation. 
$\mathcal S_{\textup{ReLU-RAF}}(X)$ seems to decrease with the context length $n$, and has in general smaller values than $\mathcal S_{\textup{RAF}}(X)$. This 
highlights the importance of the $\softmax$ function, as discussed in \emph{Step 3} of the proof sketch above. The condition $d /\log^4 d = \Omega(n)$ is required to allow \emph{step 2} to go through. We believe this to be a reasonable assumption, as the maximum context length tends to be smaller than the embedding dimension. Popular examples include BERT-Base ($n = 512$, $d = 768$), BERT-Large ($n = 512$, $d = 1024$), and the Llama-2 family ($n = 4096$, $d = 5120$).

\vspace{1em}

\begin{remark}\label{rem:rep}
The re-parameterization of the attention layer through the features $W$, which removes the dependence on queries, keys and values, does not substantially change the problem. In fact, \emph{Step 2} of the argument uses that $W$ acts as an approximate isometry on $\delta^*$. This would also hold for the product of two independent Gaussian matrices $W_Q^\top W_K$. Similarly, introducing the independent Gaussian matrix $W_V$ would not interfere with our conclusion in \emph{Step 4}. Additionally, the results obtained on the RAF model seem to extend to realistic, pre-trained, transformer architectures. In fact, in the third plot of Figure \ref{fig:RAF_S}, we provide a lower bound on the sensitivity of the BERT-Base architecture, evaluated on the abstract of this paper, when the 42nd token is modified by a perturbation $\Delta$. The blue line shows how after the first layer the embeddings are heavily modified, and how this change increases deeper in the architecture. In this case, the perturbation $\Delta$ follows from the proof idea of Theorem \ref{thm:RAF}, as it is chosen to \emph{move all the attention} towards $x_i$ (see the fourth plot in Figure \ref{fig:RAF_S}). To do so, we set $\Delta$ to be aligned with the right singular vector of $W_Q^\top W_K$ associated with the largest singular value. We do this on all the heads in the first layer separately, and then average and re-normalize.
\end{remark}

\section{Generalization on context modification}\label{sec:gen}

The study of the word sensitivity is motivated by understanding the capabilities of a model to learn to distinguish two contexts $X$ and $X^i(\Delta)$ that only differ by a word. In fact, in practice, modifying a single row/word of $X$ can lead to a significant change in the meaning, see Figure \ref{fig:att_scores}. We 
formalize the problem 
in a supervised learning setting, and characterize whether a \emph{generalized linear model} (GLM) induced by a feature map $\varphi$ generalizes over a sample $(X^i(\Delta), y_\Delta)$, after being trained on $(X, y)$. Crucially, the index $i \in [n]$ and the perturbation $\Delta$ are s.t.\ the label $y_{\Delta}$ is different from $y$ (e.g., $y\in\{-1, +1\}$ and $y_{\Delta}=-y$), namely, perturbing the $i$-th word changes the meaning of the context. 

\paragraph{Supervised learning with generalized linear models (GLMs).} Let $(\mathcal X, \mathcal Y)$ be a labelled training dataset, where $\mathcal X =(X_1, \ldots, X_N)$ contains the training data $X_i\in\mathbb R^{n\times d}$ 
and $\mathcal Y = [y_1, \ldots, y_N]^{\top}\in \{-1, 1\}^N$ the corresponding binary labels. 
The sample $(X, y)$ does not belong to $(\mathcal X, \mathcal Y)$, and it is introduced later in the training set. Let $\varphi : \R^{n \times d} \to \R^p$ be a 
feature map, and 
consider the GLM
\vspace{.6em}
\begin{equation}\label{eq:rfmodel}
    f_{\varphi}(\cdot, \theta) = \varphi(\cdot)^\top \theta,
\vspace{.5em}
\end{equation}
where 
$\theta \in \R^p$ are 
trainable parameters of the model. We define the feature matrix as $\Phi_{\varphi} := [\varphi(X_1), \ldots, \varphi(X_N)]^\top \in \R^{N \times p}$, and  
focus on the quadratic loss 
\vspace{.6em}
\begin{equation}\label{eq:loss}
    \mathcal L(\theta) := \frac{1}{N} \sum_{j = 1}^N \left( \varphi(X_j)^\top \theta - y_j \right)^2.
\vspace{.5em}
\end{equation}
Minimizing \eqref{eq:loss} with gradient descent gives (see equation (33) in \cite{bartlett21deep}) 
\vspace{.6em}
\begin{equation}\label{eq:thetastar}
    \theta^* = \theta_0 + \Phi_\varphi^+ (\mathcal Y - \Phi_\varphi \theta_0),
\vspace{.5em}
\end{equation}
where $\theta^*$ is the gradient descent solution, $\theta_0$ is the initialization, and $\Phi^+_\varphi$ is the Moore-Penrose inverse of $\Phi_\varphi$. %


Our goal is to establish whether the additional training on the sample $(X, y)$ allows the model $f_\varphi$ to generalize 
on $(X^i(\Delta), y_\Delta)$, with $ y_{\Delta} = -y$. Thus, we do \emph{not} focus on the case where $f_\varphi(X, \theta^*) = y$ and $f_\varphi(X^i(\Delta), \theta^*) = y_{\Delta} $, i.e., $f_\varphi(\cdot, \theta^*)$ already generalizes well on the two new samples.
Instead, we look at how $f_\varphi(\cdot, \theta^*)$ \emph{extrapolates} the information contained in the pair $(X, y)$ to the perturbed sample $X^i(\Delta)$. This motivates the following assumption.

\vspace{0.2em}
\begin{assumption}\label{ass:uncertainity}
There exists a parameter $\gamma \in [0, 2)$ s.t.\
\vspace{.4em}
\begin{equation}\label{eq:confidence}
    \left| f_\varphi(X^i(\Delta), \theta^*) - f_\varphi(X, \theta^*) \right| \leq \gamma.
\vspace{.3em} 
\end{equation}
\end{assumption}
The parameter $\gamma$ captures the degree over which the trained model $f_\varphi(\cdot, \theta^*)$ can distinguish between $X$ and $X^i(\Delta)$. If $\gamma = 0$,  $f_\varphi(\cdot, \theta^*)$ does not recognize any difference between $X$ and $X^i(\Delta)$; instead, if $f_\varphi(\cdot, \theta^*)$ correctly classifies the two samples $X$ and $X^i(\Delta)$ (i.e., $f_\varphi(X, \theta^*) = y$ and $f_\varphi(X^i(\Delta), \theta^*) = y_{\Delta} = -y$), then $\gamma=2$, which is beyond the scope of our analysis. 

We consider training on 
$(X, y)$ in the following two ways.

\vspace{0.2em}

\emph{(a) Fine-tuning.} First, we look at the model obtained after \emph{fine-tuning} the solution $\theta^*$ defined in \eqref{eq:thetastar} over the new sample $(X, y)$. This means that $\theta^*$ is the initialization of a gradient descent algorithm trained on this single sample. As $\left(\varphi(X)^\top\right)^+ = \varphi(X) / \norm{\varphi(X)}_2^2$, 
the fine-tuned solution is given by
\vspace{.4em}
\begin{equation}\label{eq:finetuned}
    \theta^*_f = \theta^* + \frac{\varphi(X)}{\norm{\varphi(X)}_2^2} \left( y - \varphi(X)^\top \theta^* \right).
\vspace{.3em}
\end{equation}

\emph{(b) Re-training.} Second, we re-train the model from scratch, after adding the pair $(X, y)$ to the training set. The new training set is denoted by $(\mathcal X_{r}, \mathcal Y_{r})$, where $\mathcal X_{r} =(X_1, \ldots, X_N, X)$ contains the training data  and $\mathcal Y_{r}=[y_1, \ldots, y_N, y]$ the binary labels. Thus, denoting by $\Phi_{\varphi, r} := [\varphi(X_1), \ldots, \varphi(X_N), \varphi(X)]^\top \in \R^{(N+1) \times p}$ the new feature matrix, the re-trained solution $\theta^*_r$ takes the form
\vspace{.4em}
\begin{equation}\label{eq:retrained}
    \theta^*_r = \theta_0 + \Phi_{\varphi, r}^+ (\mathcal Y_r - \Phi_{\varphi, r} \theta_0).
\vspace{.3em}
\end{equation}

The quantity of interest is the \emph{test error} on $(X^i(\Delta), y_{\Delta})$:
\vspace{.4em}
\begin{equation}\label{eq:error}
    \textup{Err}_\varphi(X^i(\Delta), \theta) := \left( f_\varphi(X^i(\Delta), \theta) - y_{\Delta} \right)^2,
\vspace{.3em}
\end{equation}
where $\theta\in\{\theta^*_f, \theta^*_r\}$ is the vector of parameters obtained either after fine-tuning or re-training.

\subsection{Random features do not generalize}\label{sec:rafgen}

By exploiting the low word sensitivity of random features, we show that both the fine-tuned and retrained solutions generalize poorly. 

\begin{theorem}\label{thm:erroronDfinetuning}
Let $\varphi_{\textup{RF}}$ be the random features map defined in \eqref{eq:rf}, with $\phi$ Lipschitz and not identically $0$, and let $f_{\textup{RF}}(\cdot, \theta^*_f)=\varphi_{\textup{RF}}(\cdot)^\top \theta^*_f$ be the corresponding model fine-tuned on the sample $(X, y)$, where $X \in \R^{n \times d}$ satisfies 
Assumption \ref{ass:d} and $\theta^*_f$ is given by \eqref{eq:finetuned}. 
Assume $k = \Omega(D)$, $|f_{\textup{RF}}(X, \theta^*)| = \bigO{1}$, 
    and that Assumption \ref{ass:uncertainity} holds with $\gamma\in [0, 2)$. Let $\textup{Err}_{\textup{RF}}(X^i(\Delta), \theta^*_f)$ be the test error of $f_{\textup{RF}}(\cdot, \theta^*_f)$ on $(X^i(\Delta), y_\Delta)$ as defined in \eqref{eq:error}. Then, for any $\Delta$ s.t.\ $\norm{\Delta}_2 \leq \sqrt d$ and any $i \in [n]$, we have
   \vspace{.4em}
   \begin{equation}\label{eq:lbft}
        \textup{Err}_{\textup{RF}}(X^i(\Delta), \theta^*_f) > (2 - \gamma)^2 - \bigO{1/\sqrt n},
    \vspace{.3em}
    \end{equation}
    with probability at least $1 - \exp(-cD)$ over $V$. 
\end{theorem}


\paragraph{Proof sketch.} The idea is to use \eqref{eq:finetuned} to obtain that
\begin{equation}\label{eq:featbody}
    f_{\textup{RF}}(X^i(\Delta), \theta^*_f) - f_{\textup{RF}}(X^i(\Delta), \theta^*) =   \frac{\varphi_{\textup{RF}}(X^i(\Delta))^\top \varphi_{\textup{RF}}(X)}{\norm{\varphi_{\textup{RF}}(X)}_2^2} \left( y - f_{\textup{RF}}(X, \theta^*) \right).
\end{equation}
Next, we note that
\vspace{.4em}
\begin{equation}\label{eq:SRF_apply}
    \left| \frac{\varphi_{\textup{RF}}(X^i(\Delta))^\top \varphi_{\textup{RF}}(X)}{ \norm{\varphi_{\textup{RF}}(X)}_2^2} - 1 \right| \leq \mathcal S_{\textup{RF}}(X).
\vspace{.3em}
\end{equation}
Theorem \ref{thm:rfS} gives that the RHS of \eqref{eq:SRF_apply} is small, which combined with \eqref{eq:featbody} implies that $f_{\textup{RF}}(X^i(\Delta), \theta^*_f)$ is close to
\vspace{.3em}
\begin{equation}\label{eq:comq}
    y+f_{\textup{RF}}(X^i(\Delta), \theta^*)- f_{\textup{RF}}(X, \theta^*).
\vspace{.3em}
\end{equation}
By upper bounding $f_{\textup{RF}}(X^i(\Delta), \theta^*)- f_{\textup{RF}}(X, \theta^*)$ via Assumption \ref{ass:uncertainity}, we obtain that \eqref{eq:comq} cannot be far from $y$. This implies that $f_{\textup{RF}}(X^i(\Delta), \theta^*_f)$ cannot be close to the correct label $y_\Delta=-y$.
The complete proof is in Appendix \ref{app:rf}. \qed


\begin{theorem}\label{thm:erroronD}
Let $\varphi_{\textup{RF}}$ be the random features map defined in \eqref{eq:rf}, with $\phi$ Lipschitz and non-linear, and let $f_{\textup{RF}}(\cdot, \theta^*_r)=\varphi_{\textup{RF}}(\cdot)^\top \theta^*_r$ be the corresponding model re-trained on the dataset $(\mathcal X_{r}, \mathcal Y_{r})$ that contains the pair $(X, y)$, thus with 
$\theta^*_r$ defined in \eqref{eq:retrained}.
Assume the training data to be sampled i.i.d.\ from a distribution $\mathcal P_X$ s.t.\  $\E_{X \sim \mathcal P_X}[X] = 0$, Assumption \ref{ass:d} holds, and the Lipschitz concentration property is satisfied. Let $N \log^3 N = o(k)$, $N \log^4 N = o(D^2)$ and $k = \Omega(D)$. Assume that $|f_{\textup{RF}}(X, \theta^*)| = \bigO{1}$ and that Assumption \ref{ass:uncertainity} holds with $\gamma\in [0, 2)$. Let $\textup{Err}_{\textup{RF}}(X^i(\Delta), \theta^*_r)$ be the test error of $f_{\textup{RF}}(\cdot, \theta^*_r)$ on $(X^i(\Delta), y_\Delta)$ defined in \eqref{eq:error}. Then, for any $\Delta$ s.t.\ $\norm{\Delta}_2 \leq \sqrt d$ and any $i \in [n]$, we have
\vspace{.4em}
\begin{equation}
    \textup{Err}_{\textup{RF}}(X^i(\Delta), \theta^*_r) > (2 - \gamma)^2 - \bigO{1/\sqrt n},
\vspace{.3em}
\end{equation}
with probability at least $1 - \exp \left( -c \log^2 N \right)$ over $V,\mathcal X_{r}$. 
\end{theorem}


\paragraph{Proof sketch.} The idea is to leverage the stability analysis in 
\cite{bombari2023stability}, which gives 
\vspace{.4em}
\begin{equation}\label{eq:featalbody}
    f_{\textup{RF}}(X^i(\Delta), \theta^*_{r}) - f_{\textup{RF}}(X^i(\Delta), \theta^*) 
    = \mathcal F_{\textup{RF}}(X, X^i(\Delta)) \left (f_{\textup{RF}}(X, \theta^*_{r}) - f_{\textup{RF}}(X, \theta^*) \right),
\vspace{.3em}
\end{equation}
where
\vspace{.4em}
\begin{equation}
\mathcal F_{{\textup{RF}}}(X, X^i(\Delta)) := \frac{\varphi_{\textup{RF}}(X^i(\Delta))^\top \Pp^\perp \varphi_{\textup{RF}}(X)}{\norm{\Pp^\perp \varphi_{\textup{RF}}(X)}_2^2}
\vspace{.3em}
\end{equation}
is the \emph{feature alignment} between $X$ and $X^i(\Delta)$ induced by $\varphi_{\textup{RF}}$ and $\Pp$ the projector over $\Span \{\varphi_{\textup{RF}}(X_1), \ldots, \varphi_{\textup{RF}}(X_N)
\}$. After some manipulations, we have 
\vspace{.4em}
\begin{equation}\label{eq:int1}
    \left| \mathcal F_{\textup{RF}}(X, X^i(\Delta)) - 1 \right| \leq \mathcal S_{\textup{RF}}(X) \frac{\norm{\varphi_{\textup{RF}}(X)}_2}{\sqrt{\evmin{K_{{\textup{RF}},r}}}},
\vspace{.3em}
\end{equation}
where $K_{\textup{RF}, r} := \Phi_{\textup{RF}, r} \Phi^\top_{\textup{RF}, r}$ is the kernel of the model. A lower bound on its smallest eigenvalue $\evmin{K_{{\textup{RF}},r}}$ follows from the fact that the kernel is well-conditioned (see Lemma \ref{lemma:evminRF}), 
which crucially relies on the assumptions on the data (i.i.d.\ and Lipschitz concentrated) and the scalings $N \log^3 N = o(k)$, $N \log^4 N = o(D^2)$. As the word sensitivity $\mathcal S_{\textup{RF}}(X)$ is upper bounded by Theorem \ref{thm:rfS}, from \eqref{eq:int1} we conclude that $\mathcal F_{\textup{RF}}(X, X^i(\Delta))$ is close to $1$.

Since $K_{\textup{RF}, r}$ is invertible, the re-trained model $f_{\textup{RF}}(\cdot, \theta^*_{r})$ interpolates the dataset $(\mathcal X_{r}, \mathcal Y_{r})$, giving that $f(X, \theta^*_{r}) = y$. Thus, as $\mathcal F_{\textup{RF}}(X, X^i(\Delta))\approx 1$, $f_{\textup{RF}}(X^i(\Delta), \theta^*_{r})$ is close to \eqref{eq:comq}, and we conclude from the same argument used for Theorem \ref{thm:erroronDfinetuning}. The complete proof is in Appendix \ref{app:rf}. \qed

In a nutshell, by exploiting the low word sensitivity of random features, Theorems \ref{thm:erroronDfinetuning}-\ref{thm:erroronD} show that, after either fine-tuning or re-training, the model does not learn to ``separate'' the predictions on 
the samples $X$ and $X^i(\Delta)$. As a consequence, the test error is lower bounded by $(2-\gamma)^2$. In fact, $\gamma$ is the distance between the predictions on $X$ and $X^i(\Delta)$ before fine-tuning/re-training (see \eqref{eq:confidence}), and the ground-truth labels have distance $2$ ($y_\Delta, y\in \{-1, 1\}$ and $y_\Delta=-y$).





While Theorem \ref{thm:erroronDfinetuning} does not require distributional assumptions on the data, Theorem \ref{thm:erroronD} considers i.i.d.\ training data, satisfying Lipschitz concentration. This property corresponds to having well-behaved tails, and it is common in the related theoretical literature \cite{bubeck2021a, tightbounds,bombari2022memorization}, see Appendix \ref{app:notation} for the formal definition and a discussion.

We remark that Assumption \ref{ass:uncertainity} requires the model $f_{\textup{RF}}(\cdot, \theta^*)$ to give a similar output when evaluated on the two new samples $X$ and $X^i(\Delta)$. Thus, we are asking if the model generalizes on $X^i(\Delta)$ \emph{only} from the additional training on $X$. Now, one could design an adversarial $\Delta$ s.t. $f(X^i(\Delta), \theta^*_{r})$ and $f(X, \theta^*_{r})$ are different from each other (so that $f_{\textup{RF}}(X^i(\Delta), \theta^*_{r}) = y_\Delta$ while $f_{\textup{RF}}(X, \theta^*_{r}) = y$), by exploiting the adversarial vulnerability of random features \cite{dohmatob2022non, dohmatob, bombari2023universal}. 
However, if we restrict the possible $\Delta$'s to those that satisfy Assumption \ref{ass:uncertainity}, Theorems \ref{thm:erroronDfinetuning} and \ref{thm:erroronD} prove that such adversarial patch cannot be found. We finally note that, when the context length $n$ is comparable or larger than the number of training samples $N$, the model becomes adversarially robust to any token modification and Assumption \ref{ass:uncertainity} automatically holds, see Appendix \ref{app:robustness} for details.

\vspace{0.2em}
\subsection{Random attention features can generalize}\label{sec:RAFgen}

Next, the behavior of random features is contrasted with that of random attention features. Let us consider the RAF model $f_{\textup{RAF}}(\cdot, \theta) := \Flat (\varphi_{\textup{RAF}}(\cdot) )^\top \theta$, where $\varphi_{\textup{RAF}}(\cdot)$ is defined in \eqref{eq:raf}. Theorem \ref{thm:RAF} proves that the word sensitivity of $\varphi_{\textup{RAF}}(\cdot)$ is large. 
This suggests that the RAF model is capable of extrapolating the information contained in $(X, y)$ to correctly classify the perturbed sample $X^i(\Delta)$. While proving a rigorous statement on a fine-tuned/re-trained RAF model remains challenging, we provide experimental evidence of this generalization capability. 

\vspace{0.2em}

Figure \ref{fig:gen} (first row) shows that, after fine-tuning on $(X, y)$, $f_{\textup{RAF}}(X^i(\Delta), \theta^*_f)$ can be close to the perturbed label $y_\Delta=-y$, even if the model before fine-tuning was unable to distinguish between $X$ and $X^i(\Delta)$. Specifically, the two central sub-plots consider the RAF model for two values of the context length $n\in \{40, 120\}$: here, the loss on the perturbed sample can be close to $0$, even when the parameter $\gamma$ in \eqref{eq:confidence} is close to $0$, i.e., $X$ and $X^i(\Delta)$ were indistinguishable before fine-tuning; in general, the test error $\textup{Err}_{\textup{RAF}}(X^i(\Delta), \theta^*_f)$ is often smaller than the lower bound of $(2-\gamma)^2$ (dashed black line), which holds for random features.
\vspace{0.2em}

The left sub-plot considers the RF model for $n=40$: here, the loss on the perturbed sample is close to $0$ only if the model before fine-tuning was already able to perfectly distinguish between $X$ and $X^i(\Delta)$, i.e., $\gamma$ is not far from $2$; 
in general, the test error $\textup{Err}_{\textup{RF}}(X^i(\Delta), \theta^*_f)$ always respects the lower bound of $(2-\gamma)^2$ proved in Theorem \ref{thm:erroronDfinetuning}. Finally, the right sub-plot considers the ReLU-RAF model (which replaces the $\softmax$ with a ReLU activation, as described at the end of Section \ref{sec:raf}) for $n=120$: here, even if the lower bound of Theorem  \ref{thm:erroronDfinetuning} is often violated, the model still cannot reach small error unless $\gamma$ is large, i.e., $X$ and $X^i(\Delta)$ could be distinguished already before fine-tuning. This confirms the impact of the $\softmax$ on the capability of attention layers to understand the context. Analogous results hold when models are re-trained (instead of being fine-tuned), as reported in the second row of the same figure.

\vspace{0.2em}

In a nutshell, our results show that, when a RAF model is not able to distinguish two points with opposite labels (that only differ in one word), fine-tuning or retraining on one of these points allows the loss on also the other point to decrease. In contrast, this is not the case for the RF model, where the loss on the second point can be lower bounded according to Theorems \ref{thm:erroronDfinetuning} and \ref{thm:erroronD}. This is shown in Figure \ref{fig:gen}, where most of the points for the RAF model are below the dashed line representing the lower bound for the RF model.

\begin{figure*}[!t]
\centering
\includegraphics[width=\textwidth]{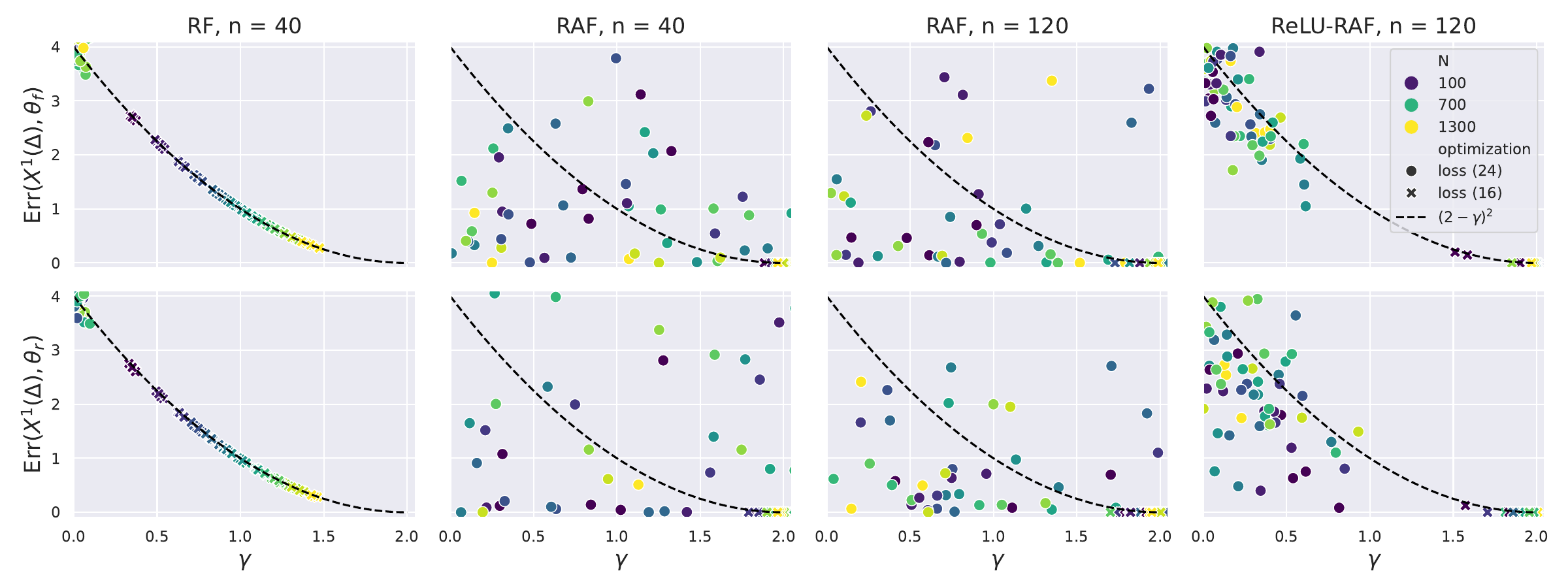}
\vspace{-1.5em}
\caption{Test error (as defined in \eqref{eq:error} taking $i=1$) for the RF (left subplot), RAF (two central sub-plots) and ReLU-RAF (right subplot) maps, as a function of the smallest $\gamma$ s.t.\ Assumption \ref{ass:uncertainity} holds. The first (resp.\ second) row considers the fine-tuned solution $\theta^*_f$ (resp.\ re-trained solution $\theta^*_r$). Every sub-plot has a fixed embedding dimension $d = 768$, and context length $n\in \{40, 120\}$, taking the first $n$ token embeddings for each sample. Different colors correspond to a different number of training samples $N \in \{100, 700, 1300\}$. Every point in the scatter-plots is an independent simulation where $(X, y)$ and $(\mathcal X, \mathcal Y)$ are the BERT-Base embeddings of a random subset of the imdb dataset (after pre-processing to fulfill Assumption \ref{ass:d}).
Circular markers correspond to obtaining $\Delta$ via gradient descent optimization of the losses in \eqref{eq:lossf}; cross markers correspond to minimizing directly the test error in \eqref{eq:error}.}
\label{fig:gen}
\end{figure*}

\subsection{Experimental details}
The experiments of Figure \ref{fig:gen} are performed on multiple independent trials, for different choices of the training data. 
We report in cross markers the results obtained by choosing $\Delta$ after optimizing the test error in \eqref{eq:error} via gradient descent. While this approach directly minimizes the metric of interest, it results in low test error only when $\gamma$ is rather large, regardless of the model taken into account (RF, RAF, or ReLU-RAF). In contrast, optimizing a different loss controls the value of $\gamma$, while still achieving small error for RAF (and, to a smaller extent, ReLU-RAF). We report in circular markers the results obtained by minimizing the following two losses (respectively, for fine-tuning and re-training):
\begin{equation}\label{eq:lossf}   
    \ell_{\theta^*_f}(\Delta) := \left( \frac{\varphi_{\textup{RAF}}(X^i(\Delta))^\top \varphi_{\textup{RAF}}(X)}{\norm{\varphi_{\textup{RAF}}(X)}_2^2} + 1 \right)^2, \qquad
    \ell_{\theta^*_r}(\Delta) := \left( \mathcal F_{{\textup{RAF}}}(X, X^i(\Delta)) + 1 \right)^2.
\end{equation}
This choice is suggested by \eqref{eq:featbody} and \eqref{eq:featalbody} which, after assuming for simplicity that $f_{\textup{RAF}}(X^i(\Delta), \theta^*) = f_{\textup{RAF}}(X, \theta^*) = 0$, can be re-written as 
\begin{equation}\label{eq:heur}
    f_{\textup{RAF}}(X^i(\Delta), \theta^*_f)  = \frac{\varphi_{\textup{RAF}}(X^i(\Delta))^\top \varphi_{\textup{RAF}}(X)}{\norm{\varphi_{\textup{RAF}}(X)}_2^2} \, y, \qquad
    f_{\textup{RAF}}(X^i(\Delta), \theta^*_r) = \mathcal F_{{\textup{RAF}}}(X, X^i(\Delta)) \, y.
\end{equation}
Achieving small error means that the LHS of \eqref{eq:heur} is close to $-y$, which corresponds to making the losses in \eqref{eq:lossf} small. Further numerical results and comparisons between different optimization algorithms for finding $\Delta$ are in Appendix \ref{app:exp}.

\section{Conclusions}

This work provides a formal characterization of the fundamental difference between fully connected and attention layers. To do so, we consider the prototypical setting of random features and study the \emph{word sensitivity}, which captures how the output of a map changes after perturbing a single row/word of the input. On the one hand, the sensitivity of standard random features decreases with the context length and, in order to obtain to a significant change in the output of the map, a constant fraction of the words needs to be perturbed. On the other hand, the sensitivity of random attention features is large, regardless of the context length, thus indicating the suitability of attention layers for NLP tasks. These bounds on the word sensitivity translate into formal negative generalization results for random features, which are contrasted by positive empirical evidence of generalization for the attention layer. 

Our analysis allows the 
perturbations to be any (bounded) vector in the embedding space. Taking the tokenization process (and, hence, the discrete nature of the textual samples) into account offers an exciting avenue for future work.

\section*{Acknowledgements}

The authors were partially supported by the 2019 Lopez-Loreta prize, and they would like to thank Mohammad Hossein Amani, Lorenzo Beretta, and Cl\'{e}ment Rebuffel for helpful discussions.

{
\small

\bibliographystyle{plain}
\bibliography{bibliography.bib}

}

\newpage
\appendix
\onecolumn

\section{Additional notation}\label{app:notation}

Given a sub-Gaussian random variable, let $\|X\|_{\psi_2} = \inf \{ t>0 \,\,: \,\,\mathbb E[\exp(X^2/t^2)] \le 2 \}$, see Section 2.5 of \cite{vershynin2018high}.
Given a sub-exponential random variable $X$, let $\|X\|_{\psi_1} = \inf \{ t>0 \,\,: \,\,\mathbb E[\exp(|X|/t)] \le 2 \}$, see Section 2.7 of \cite{vershynin2018high}).
We recall the property that, if $X$ and $Y$ are scalar random variables, then $\subEnorm{XY} \leq \subGnorm{X} \subGnorm{Y}$, see Lemma 2.7.7 of \cite{vershynin2018high}.

We use the term \emph{standard Gaussian vector} in $\R^d$ to indicate a vector $\rho$ such that $\rho_i \distas{}_{\textup{i.i.d.}} \mathcal N(0, 1)$. We recall that the maximum of $n$ Gaussian (not necessarily independent) random variables is smaller than $\log n$ with probability at least $1 - \exp(c \log^2 n)$, see, e.g., Section 1.4 of \cite{rigollet2023highdimensional}.

Given a vector $u \in \R^d$, we denote by $u_i$ its $i$-th component. Given a matrix $A \in \R^{n \times d}$, we denote by $[A]_{i:}$ its $i$-th row, by $[A]_{:j}$ its $j$-th column, and by $[A]_{ij}$ its entry at position $(i, j)$.

We say that a random variable or vector respects the Lipschitz concentration property if there exists an absolute constant $c > 0$ such that, for every Lipschitz continuous function $\varphi: \RR^d \to \RR$, we have $\E |\varphi(X)| < + \infty$ and for all $t>0$,
\begin{equation}\label{eq:deflipconc}
    \PP\left(\abs{\varphi(x)- \E_X [\varphi(x)]}>t\right) \leq 2e^{-ct^2 / \norm{\varphi}_{\Lip}^2}.
\end{equation}

The family of Lipschitz concentrated distributions covers a number of important cases, e.g., standard Gaussian \cite{vershynin2018high}, uniform on the sphere and on the unit (binary or continuous) hypercube \cite{vershynin2018high}, or data obtained via a Generative Adversarial Network (GAN) \cite{seddik2020random}.

\section{Proofs for random features}\label{app:rf}

In this section, we provide the proofs for our results on the random features model. Thus, we will drop the sub-script ``RF'' in all the quantities of this section, for the sake of a cleaner notation. We consider a single textual data-point $X = [x_1, x_2, ..., x_n]^\top \in \R^{n \times d}$ that satisfies Assumption \ref{ass:d}. We consider the random features model defined in \eqref{eq:rf}, i.e., 
\begin{equation}
    \varphi(X) = \phi (V \Flat(X)),
\end{equation}
where $V_{i,j} \distas{}_{\rm i.i.d.}\mathcal{N}(0, 1 / D)$, $D = nd$ and $\phi$ is the activation function, applied component-wise to the pre-activations $V \Flat(X)$.

Further in the section, we will investigate the generalization capabilities of the RF model $f(\cdot, \theta) := \varphi(\cdot)^\top \theta$ on token modification. We use the notation $(\mathcal X, \mathcal Y)$ to indicate the original labelled training dataset, with $\mathcal X =(X_1, \ldots, X_N)$, $X_i \in \R^{n \times d}$ and $\mathcal Y = [y_1, \ldots, y_N]^\top \in \{-1, +1\}^N$. We will use the short-hand $\Phi := [\varphi(X_1), ... \varphi(X_N)]^\top \in \R^{N \times p}$ for the feature matrix and $K:= \Phi \Phi^\top$ for the kernel. According to \eqref{eq:thetastar}, minimizing the quadratic loss over this dataset returns the parameters
\begin{equation}
    \theta^* = \theta_0 + \Phi^+ (\mathcal Y - \Phi \theta_0).
\end{equation}
We consider the test error on the modified sample $X^i(\Delta)$ given by (see also \eqref{eq:error}) 
\begin{equation}
    \textup{Err}(X^i(\Delta), \theta) := \left( f(X^i(\Delta), \theta) - y_{\Delta} \right)^2.
\end{equation}
We will investigate this quantity for both the fine-tuned and the re-trained model. In particular, the new solution obtained after fine-tuning over the sample $(X, y)$ gives
\begin{equation}
    \theta^*_f = \theta^* + \frac{\varphi(X)}{\norm{\varphi(X)}_2^2} \left( y - \varphi(X)^\top \theta^* \right),
\end{equation}
while retraining the model with initialization $\theta_0$ on the new dataset $(\mathcal X_{r}, \mathcal Y_{r})$ 
returns
\begin{equation}
    \theta^*_r = \theta_0 + \Phi_{r}^+ (\mathcal Y_r - \Phi_{r} \theta_0),
\end{equation}
where we denote by $\Phi_{r} := [\varphi(X_1), ... \varphi(X_N), \varphi(X)]^\top \in \R^{(N+1) \times p}$ the new feature matrix. The corresponding kernel is denoted by $K_r = \Phi_r \Phi_r^\top$.

The outline of this section is the following:
\begin{enumerate}
    \item We use Lemma \ref{lemma:0rf} to upper bound the numerator of the word sensitivity $\mathcal S(X)$ defined in \eqref{eq:sensitivity}, which readily allows to prove the desired result Theorem \ref{thm:rfS}. 
    \item We prove Theorem \ref{thm:erroronDfinetuning}, where we lower bound $\textup{Err}(X^i(\Delta), \theta^*_f)$ for the fine-tuned solution $\theta^*_f$ as a function of $\gamma$.
    \item We prove Theorem \ref{thm:erroronD}, where we lower bound $\textup{Err}(X^i(\Delta), \theta^*_r)$ for the retrained solution $\theta^*_r$ as a function of $\gamma$. This result requires additional assumptions and analysis:
    \begin{itemize}
        \item We report in our notation Lemma 4.1 from \cite{bombari2023stability}, which defines the \emph{feature alignment} $\mathcal F(X, X^i(\Delta))$ between the two samples $X$ and $X^i(\Delta)$.
        \item In Lemma \ref{lemma:evminRF}, we exploit the additional assumptions to prove a lower bound on the smallest eigenvalue of the kernel $\evmin{K_{r}} = \Omega(k)$. 
        \item In Lemma \ref{lemma:Falis1}, we show that $\left| \mathcal F(X, X^i(\Delta)) - 1 \right| = \bigO{1/\sqrt n}$ with high probability. This step is useful as here this term assumes the role previously taken by $\mathcal S(X)$.
    \end{itemize}
\end{enumerate}

\begin{lemma}\label{lemma:0rf}
Let $\varphi(X)$ be the random feature map defined in \eqref{eq:rf}, with $\phi$ Lipschitz and $k = \Omega(D)$. Let $\Delta \in \R^d$ be such that $\norm{\Delta}_2 \leq \sqrt d$. Then, for every $i \in [n]$, we have
\begin{equation}
    \norm{\varphi(X^i(\Delta)) - \varphi(X)}_2 = \bigO{\sqrt{\frac{k}{n}}},
\end{equation}
with probability at least $1 - \exp(-c D)$ over $V$.
\end{lemma}
\begin{proof}
Let's condition on the event
\begin{equation}
    \opnorm{V} = \bigO{\sqrt{\frac{D + k}{D}}},
\end{equation}
which happens with probability at least $1 - \exp(-c_1 D)$ over $V$, by Theorem 4.4.5 of \cite{vershynin2018high}. Thus, for every $i$, we have 
\begin{equation}\label{eq:onlythingforrf}
\begin{aligned}
    \norm{\varphi(X^i(\Delta)) - \varphi(X)}_2 &= \norm{\phi(V \Flat{(X^i(\Delta))}) - \phi (V \Flat{(X)})}_2 \\
    &\leq M \norm{V (\Flat{(X^i(\Delta))} - \Flat{(X)})}_2 \\
    &\leq M \opnorm{V} \norm{\Delta}_2 \\
    &\leq C_1 \sqrt{\frac{D + k}{D}} \sqrt d \\
    & = C_1 \sqrt{\frac{D + k}{n}} \\
    & \leq C_2 \sqrt{\frac{k}{n}},
\end{aligned}
\end{equation}
where the first inequality comes from the Lipschitz continuity of $\phi$, and the last step is a consequence of $k = \Omega(D)$. 
\end{proof}

\paragraph{Theorem \ref{thm:rfS}} Let $\varphi(X)$ be the random feature map defined in \eqref{eq:rf}, where $\phi$ is Lipschitz and not identically 0. Let $X \in \R^{n \times d}$ be a generic input sample s.t.\ Assumption \ref{ass:d} holds, and assume $k = \Omega(D)$. Let $\mathcal S(X)$ denote the the word sensitivity defined in \eqref{eq:sensitivity}. Then, we have
\begin{equation}
  \mathcal S(X) = \bigO{1/\sqrt n} = o(1),
\end{equation}
with probability at least $1 - \exp(-c D)$ over $V$.
\begin{proof}
As $\phi$ is Lipschitz and non-0, we can apply the result in Lemma C.3 of \cite{bombari2023stability}, getting
\begin{equation}
    \norm{\varphi(X)}_2 = \Theta(\sqrt k),
\end{equation}
with probability at least $1 - \exp(-c_1 D)$ over $V$. Thus, the thesis readily follows from Lemma \ref{lemma:0rf}.
\end{proof}

\begin{proof}[Proof of Remark \ref{rem:constant}.]
The only difference with respect to the argument for Theorem \ref{thm:rfS} is in \eqref{eq:onlythingforrf}. Now, $\Delta$ is replaced by a new set of $m$ perturbations $\Delta_1, \dots ,\Delta_m$. Thus, the modified context takes the form
\begin{equation}
    X\left( \Delta_1, \dots ,\Delta_m \right) = X + \sum_{j=1}^m e_{i_j} \Delta_j^\top,
\end{equation}
where $\{e_{i_j}\}_{j\in [m]}$ represent different elements of the canonical basis. Thus, we have
\begin{equation}
    \norm{\Flat\left( X\left( \Delta_1, \dots ,\Delta_m \right)\right) - \Flat\left(X\right)}_2 =   \norm{\Flat\left( \sum_{j=1}^m e_{i_j} \Delta_j^\top \right)}_2 = \norm{\sum_{j=1}^m e_{i_j} \Delta_j^\top}_F.
\end{equation}
Since the $e_{i_j}$'s are all distinct (as we are modifying $m$ different words), we obtain
\begin{equation}
    \norm{\sum_{j=1}^m e_{i_j} \Delta_j^\top}_F^2 = \sum_{j=1}^m \norm{\Delta_j}_2^2 \leq md,
\end{equation}
where in the last step we use that $\norm{\Delta_j}_2 \leq \sqrt d$ for all $j \in [m]$.

This allows to replace the $\sqrt d$ in the fourth line of \eqref{eq:onlythingforrf} with $\sqrt{md}$, increasing the final bound in the statement of Lemma \ref{lemma:0rf} by a factor $\sqrt m$. Thus, the upper bound on the sensitivity is given by $\sqrt{m / n}=o(1)$, which concludes the proof.
\end{proof}

\paragraph{Theorem \ref{thm:erroronDfinetuning}} Let $f(\cdot, \theta^*_f)$ be the RF model fine-tuned on the sample $(X, y)$, where $\phi$ in \eqref{eq:rf} is Lipschitz and not identically 0, $X \in \R^{n \times d}$ is a generic sample s.t.\ Assumption \ref{ass:d} holds and $\theta^*_f$ is defined in \eqref{eq:finetuned}. Assume $k = \Omega(D)$, $|f(X, \theta^*)| = \bigO{1}$, and that Assumption \ref{ass:uncertainity} holds with $\gamma\in [0, 2)$. Let $\textup{Err}(X^i(\Delta), \theta^*_f)$ be the test error of the model on the sample $X^i(\Delta)$ defined in \eqref{eq:error}. Then, for any $\Delta$ such that $\norm{\Delta}_2 \leq \sqrt d$ and any $i \in [n]$, we have
\begin{equation}
    \textup{Err}(X^i(\Delta), \theta^*_f) > (2 - \gamma)^2 - \bigO{1/\sqrt n},
\end{equation}
with probability at least $1 - \exp(-c D)$ over $V$. 
\begin{proof}
    We have
    \begin{equation}
        f(X^i(\Delta), \theta^*_f) = \varphi(X^i(\Delta))^\top \theta^*_f = \varphi(X^i(\Delta))^\top \theta^* +  \frac{\varphi(X^i(\Delta))^\top \varphi(X)}{\norm{\varphi(X)}_2^2} \left( y - \varphi(X)^\top \theta^* \right),
    \end{equation}
    where the second step is justified by \eqref{eq:finetuned}.
    As $\left| \varphi(X)^\top \theta^* - \varphi(X^i(\Delta))^\top \theta^* \right| = \left| f(X, \theta^*) - f(X^i(\Delta), \theta^*) \right| \leq \gamma$ by Assumption \ref{ass:uncertainity}, we can write
    \begin{equation}
    \begin{aligned}
        \left| f(X^i(\Delta), \theta^*_f) - y \right| &= \left|f(X^i(\Delta), \theta^*) - f(X, \theta^*) 
        + \left( \frac{\varphi(X^i(\Delta))^\top \varphi(X)}{\norm{\varphi(X)}_2^2} - 1 \right)\left( y  - f(X, \theta^*)\right) \right|\\
        &\leq \gamma + \left| \frac{\varphi(X^i(\Delta))^\top \varphi(X)}{\norm{\varphi(X)}_2^2} - 1 \right| \left|  y  - f(X, \theta^*) \right|.
    \end{aligned}
    \end{equation}
    By Cauchy-Schwartz inequality, we have
    \begin{equation}
        \left| \frac{\varphi(X^i(\Delta))^\top \varphi(X)}{\norm{\varphi(X)}_2^2} - 1 \right| = \left| \frac{\left(\varphi(X^i(\Delta)) - \varphi(X) \right)^\top \varphi(X)}{\norm{\varphi(X)}_2^2}\right| \leq  \frac{ \norm{\varphi(X^i(\Delta)) - \varphi(X) }_2 \norm{\varphi(X)}_2}{\norm{\varphi(X)}_2^2} \leq \mathcal S(X).
    \end{equation}
    By Theorem \ref{thm:rfS}, we have that $\mathcal S(X) = \bigO{1/\sqrt n}$ with probability at least $1 - \exp(-c D)$ over $V$. Conditioning on such high probability event, we can write
    \begin{equation}
    \begin{aligned}
        \left| f(X^i(\Delta), \theta^*_f) - y_\Delta \right| &\geq \left| y_\Delta - y \right| - \left| f(X^i(\Delta), \theta^*_f) - y \right| \\
        &\geq \left| y_\Delta - y \right| - \gamma - \mathcal S(X) \left|  y  - f(X, \theta^*) \right| \\
        &= 2 - \gamma - S(X) \left|  y  - f(X, \theta^*) \right| \\
        &= 2 -  \gamma - \bigO{1/\sqrt n},
    \end{aligned}
    \end{equation}
    where the third step is a consequence of $|y| = |y_\Delta| = 1$, with $y = - y_\Delta$, and the fourth step comes from $|f(X, \theta^*)| = \bigO{1}$. Thus, we can conclude
    \begin{equation}
        \textup{Err}(X^i(\Delta), \theta^*_f) = \left( f(X^i(\Delta), \theta^*_f) - y_\Delta \right)^2 \geq  \left( 2 - \gamma - \bigO{1/\sqrt n} \right)^2 > (2 - \gamma)^2 - \bigO{1/\sqrt n}.
    \end{equation}
\end{proof}

\paragraph{Lemma 4.1 \cite{bombari2023stability}} Let the kernel $K_{r} \in \R^{(N+1) \times (N+1)}$ be invertible, and let $\Pp \in \R^{k \times k}$ be the projector over $\Span \{\Rows (\Phi) \}$. Let us denote by
\begin{equation}\label{eq:featalign}
\mathcal F(X, X^i(\Delta)) := \frac{\varphi(X^i(\Delta))^\top \Pp^\perp \varphi(X)}{\norm{\Pp^\perp \varphi(X)}_2^2}
\end{equation}
the \emph{feature alignment} between $X$ and $X^i(\Delta)$. Then, we have
\begin{equation}\label{eq:lemmaproj}
f(X^i(\Delta), \theta^*_{r}) - f(X^i(\Delta), \theta^*) = \mathcal F(X, X^i(\Delta)) \left (f(X, \theta^*_{r}) - f(X, \theta^*) \right).
\end{equation}
Notice that
\begin{equation}\label{eq:denevmin}
\norm{\Pp^\perp \varphi(X)}^2_2 \geq \evmin{K_r} > 0
\end{equation}
is directly implied by the invertibility of $K_{r}$, as shown in Lemma B.1 from \cite{bombari2023stability}.

\begin{lemma}\label{lemma:evminRF}
Let $\phi$ be a non-linear, Lipschitz function. Let all the training data in $\mathcal X_{r}$ be sampled i.i.d.\ according to a distribution $\mathcal P_X$ s.t.\ $\E_{X \sim \mathcal P_X}[X] = 0$, Assumption \ref{ass:d} holds, and the Lipschitz concentration property is satisfied. Let $N \log^3 N = o(k)$ and $N \log^4 N = o(D^2)$. Then, we have
\begin{equation}
    \evmin{K_{r}} = \Omega(k),
\end{equation}
with probability at least $1 - \exp \left( -c \log^2 N \right)$ over $V$ and $\mathcal X_r$. 
\end{lemma}
\begin{proof}
The desired result follows from Lemma D.2 in \cite{bombari2023stability}. Notice that for their argument to go through, they need their Lemma D.1 to hold, which requires our assumptions on the data distribution $\mathcal P_X$, and Assumption \ref{ass:d}. They further require the scalings $N = o(D^2 / \log^4 D)$ and $N \log^4 N = o(D^2)$, which are both given by our assumption $N \log^4 N = o(D^2)$. Finally, in their Lemma D.2 they require the activation function $\phi$ to be Lipschitz and non-linear, and the over-parameterized setting $N \log^3 N = o(k)$.
\end{proof}

\begin{lemma}\label{lemma:Falis1}
Let $\phi$ be a non-linear, Lipschitz function. Let all the training data in $\mathcal X_{r}$ be sampled i.i.d. according to a distribution s.t.\ $\E_{X \sim \mathcal P_X}[X] = 0$, Assumption \ref{ass:d} holds, and the Lipschitz concentration property is satisfied. Let $N \log^3 N = o(k)$, $N \log^4 N = o(D^2)$ and $k = \Omega(D)$. Let $\mathcal F(X, X^i(\Delta))$ be defined as in \eqref{eq:featalign}. Then, we have
\begin{equation}
    \left| \mathcal F(X, X^i(\Delta)) - 1 \right| = \bigO{\frac{1}{\sqrt n}},
\end{equation}
with probability at least $1 - \exp \left( -c \log^2 N \right)$ over $\mathcal X_r$ and $V$.
\end{lemma}
\begin{proof}
By Cauchy-Schwartz inequality, we have
\begin{equation}\label{eq:csfalign}
\begin{aligned}
\left| \mathcal F(X, X^i(\Delta)) - 1 \right| &= \left| \frac{\left( \varphi(X^i(\Delta)) - \varphi(X) \right)^\top \Pp^\perp \varphi(X)}{\norm{\Pp^\perp \varphi(X)}_2^2} \right| \\
&\leq  \frac{ \norm{\varphi(X^i(\Delta)) - \varphi(X) }_2}{\norm{\varphi(X)}_2}  \frac{\norm{\varphi(X)}_2}{\norm{\Pp^\perp \varphi(X)}_2} \\
&\leq \mathcal S(X) \frac{\norm{\varphi(X)}_2}{\sqrt{\evmin{K_r}}},
\end{aligned}
\end{equation}
where the last step is a consequence of \eqref{eq:denevmin}. As $\phi$ is Lipschitz and non-0, we can apply the result in Lemma C.3 of \cite{bombari2023stability}, getting
\begin{equation}
    \norm{\varphi(X)}_2 = \Theta(\sqrt k),
\end{equation}
with probability at least $1 - \exp(-c_1 D)$ over $V$. Due to Lemma \ref{lemma:evminRF}, we can also write
\begin{equation}
    \evmin{K_{r}} = \Omega(k),
\end{equation}
with probability at least $1 - \exp \left( -c_2 \log^2 N \right)$ over $V$ and $\mathcal X_r$. Thus, \eqref{eq:csfalign} promptly gives
\begin{equation}
    \left| \mathcal F(X, X^i(\Delta)) - 1 \right| \leq \mathcal S(X) \frac{\norm{\varphi(X)}_2}{\sqrt{\evmin{K_r}}} = \bigO{\mathcal S(X)} = \bigO{\frac{1}{\sqrt n}},
\end{equation}
where the last step comes from Theorem \ref{thm:rfS}, and holds with probability at least $1 - \exp(-c_3 D)$. This gives the desired result.
\end{proof}

\paragraph{Theorem \ref{thm:erroronD}} Let $f(\cdot, \theta^*_r)$ be the RF model re-trained on the dataset $(\mathcal X_{r}, \mathcal Y_{r})$ that contains the pair $(X, y)$, thus with $\theta^*_r$ defined in \eqref{eq:retrained}. Let $\phi$ in \eqref{eq:rf} be a non-linear, Lipschitz function. Assume the training data to be sampled i.i.d.\ from a distribution $\mathcal P_X$ s.t.\ $\E_{X \sim \mathcal P_X}[X] = 0$, Assumption \ref{ass:d} holds, and the Lipschitz concentration property is satisfied. Let $N \log^3 N = o(k)$, $N \log^4 N = o(D^2)$ and $k = \Omega(D)$. Assume that $|f(X, \theta^*)| = \bigO{1}$ and that Assumption \ref{ass:uncertainity} holds with $\gamma\in [0, 2)$. Let $\textup{Err}(X^i(\Delta^*), \theta^*_r)$ be the test error of the model on the sample $X^i(\Delta)$ defined in \eqref{eq:error}. Then, for any $\Delta$ s.t.\ $\norm{\Delta}_2 \leq \sqrt d$ and any $i \in [n]$, we have
\begin{equation}
    \textup{Err}(X^i(\Delta^*), \theta^*_r) > (2 - \gamma)^2 - \bigO{1/\sqrt n},
\end{equation}
with probability at least $1 - \exp \left( -c \log^2 N \right)$ over $V$ and $\mathcal X_{r}$. 
\begin{proof}
Let's condition on $K_{r}$ being invertible, which by Lemma \ref{lemma:evminRF} happens with probability at least $1 - \exp \left( -c_1 \log^2 N \right)$ over $V$ and $\mathcal X_r$. Then, by \eqref{eq:lemmaproj}, we have
\begin{equation}
    f(X^i(\Delta), \theta^*_{r}) - f(X^i(\Delta), \theta^*) = \mathcal F(X, X^i(\Delta)) \left (y - f(X, \theta^*) \right),
\end{equation}
since $f(\cdot, \theta^*_r)$ fully interpolates the training data, thus giving $f(X, \theta^*_r) = y$.
As $\left| \varphi(X)^\top \theta^* - \varphi(X^i(\Delta))^\top \theta^* \right| = \left| f(X, \theta^*) - f(X^i(\Delta), \theta^*) \right| \leq \gamma$ by Assumption \ref{ass:uncertainity}, we can write
\begin{equation}
\begin{aligned}
    \left| f(X^i(\Delta), \theta^*_r) - y \right| &= \left|f(X^i(\Delta), \theta^*) - f(X, \theta^*) 
    + \left( \mathcal F(X, X^i(\Delta)) - 1 \right)\left( y  - f(X, \theta^*)\right) \right|\\
    &\leq \gamma + \left| \mathcal F(X, X^i(\Delta)) - 1 \right| \left|  y  - f(X, \theta^*) \right|.
\end{aligned}
\end{equation}
By Lemma \ref{lemma:Falis1} we have that $\left| \mathcal F(X, X^i(\Delta)) - 1 \right| = \bigO{1/\sqrt n}$ with probability at least $1 - \exp \left( -c \log^2 N \right)$ over $\mathcal X_r$ and $V$. Conditioning on such high probability event, we can write
\begin{equation}
\begin{aligned}
    \left| f(X^i(\Delta), \theta^*_r) - y_\Delta \right| &\geq \left| y_\Delta - y \right| - \left| f(X^i(\Delta), \theta^*_r) - y \right| \\
    &\geq \left| y_\Delta - y \right| - \gamma - \left| \mathcal F(X, X^i(\Delta)) - 1 \right| \left|  y  - f(X, \theta^*) \right| \\
    &= 2 - \gamma - \left| \mathcal F(X, X^i(\Delta)) - 1 \right| \left|  y  - f(X, \theta^*) \right| \\
    &= 2 - \gamma - \bigO{1/\sqrt n},
\end{aligned}
\end{equation}
where the third step is a consequence of $|y| = |y_\Delta| = 1$, with $y = - y_\Delta$, and the fourth step comes from $|f(X, \theta^*)| = \bigO{1}$. Thus, we can conclude
\begin{equation}
    \textup{Err}(X^i(\Delta), \theta^*_r) = \left( f(X^i(\Delta), \theta^*_r) - y_\Delta \right)^2 \geq  \left( 2 -  \gamma - \bigO{1/\sqrt n} \right)^2 > (2 - \gamma)^2 - \bigO{1/\sqrt n}.
\end{equation}
\end{proof}



\section{Proofs for deep random features}\label{app:drf}

In this section, we provide the proofs for our results on the deep random features model. We consider a single textual data-point $X = [x_1, x_2, ..., x_n]^\top \in \R^{n \times d}$ that satisfies Assumption \ref{ass:d}. We consider the deep random features model defined in \eqref{eq:drf}, i.e.,
\begin{equation}
    \varphi_{\textup{DRF}}(X) := \phi(V_L \phi( V_{L-1} (... V_2 \phi (V_1 \Flat X)  ... ))),
\end{equation}
where $\phi: \R \to \R$ is the non-linearity applied component-wise at each layer, and $V_l \in \R^{D \times D}$ are the random weights at layer $l$, sampled independently and such that $[V_1]_{i,j} \distas{}_{\rm i.i.d.}\mathcal{N}(0, \beta / D)$ and $[V_l]_{i,j} \distas{}_{\rm i.i.d.}\mathcal{N}(0, \beta / k)$ for $l>1$. We set $\beta$ according to He's (or Kaiming) initialization \cite{He2015}, i.e.,
\begin{equation}\label{eq:heapp}
    \E_{\rho \sim \mathcal N(0, \beta)} \left[ \phi^2( \rho ) \right] = 1.
\end{equation}
Thus, we require the activation function $\phi$ to guarantee at least one value $\beta$ for which the previous equation is respected. We will consider $\beta$ to be a positive constant dependent only on the activation $\phi$. 

Let's introduce the shorthands
\begin{equation}\label{eq:drffeatures}
\begin{aligned}
    \varphi_0(X) &= \Flat(X), \\
    \varphi_l(X) &= \phi \left( V_l \, \varphi_{l-1}(X) \right), \qquad \textup{for } l \in [L].
\end{aligned}
\end{equation}
The outline of this section is the following:
\begin{enumerate}
    \item In Lemma \ref{lemma:concentrationnorminduction}, we prove that at every layer, the norm of the features $\norm{\varphi_{l}(X)}_2$, with $l > 1$, concentrates to $\sqrt k$. This is the step where He's initialization is necessary.
    \item In Lemma \ref{lemma:0drf}, we show that $\norm{\varphi_{l}(X^i(\Delta)) - \varphi_{l}(X)}_2$ can be upper bounded by a term that grows exponentially with the depth of the layer $l$.
    \item In Theorem \ref{thm:RAF}, we upper bound $\mathcal S_{\textup{DRF}}$, concluding the argument.
\end{enumerate}

\begin{lemma}\label{lemma:concentrationnorminduction}
Let $\varphi_l(X)$ be defined in \eqref{eq:drffeatures}, and let $\phi$ be a Lipschitz function such that \eqref{eq:He} admits at least one solution $\beta$. Let $X \in \R^{n \times d}$ be a generic input sample such that Assumption \ref{ass:d} holds, and let $k = \Theta(D)$. Then, for every $l > 0$, we have
\begin{equation}
    \left| \norm{\varphi_{l}(X)}_2 - \sqrt k \right| \leq e^{Cl} \log D,
\end{equation}
with probability at least $1 - 2 L \exp (-c \log^2 D)$ over $\{ V_l \}_{l = 1}^L$. 
\end{lemma}
\begin{proof}
By Assumption \ref{ass:d} the statement trivially holds for $l=0$, if we replace $\sqrt k$ with $\sqrt d$. Let's consider this the base case of an induction argument to prove the statement for $l \in [L]$. 

Thus, using the notation $k_0 = D$ and $k_l = k$ for $l \in [L]$, the inductive hypothesis becomes
\begin{equation}\label{eq:indhyp}
    \left| \norm{\varphi_{l - 1}(X)}_2 - \sqrt {k_{l-1}} \right| \leq  e^{C(l-1)} \log D,
\end{equation}
with probability at least $1 - 2 (l - 1) \exp (-c \log^2 D)$ on $\{ V_m \}_{m = 1}^{l-1}$. Let's condition on this high probability event and on $\opnorm{V_l} \leq C_1$ until the end of the proof. By Theorem 4.4.5 of \cite{vershynin2018high} and since $k = \Theta(D)$, there exists $C_1$ large enough and independent from $l$, such that this holds with probability at least $1 - 2 \exp (-c_1 k)$ over $V_l$. We therefore aim to prove the thesis for
\begin{equation}\label{eq:Cind}
    e^C := \max \left( 1, M C_1 + 1 \right),
\end{equation}
where we denote by $M$ the Lipschitz constant of $\phi$. \eqref{eq:Cind} explicitely shows that $C$ is a natural constant dependent only on the activation function $\phi$.
    
We have
\begin{equation}
    \varphi_{l}(X) = \phi(V_l \, \varphi_{l-1}(X) )
    = \phi \left( V_l \frac{\sqrt {k_{l-1}} \, \varphi_{l-1}(X)}{\norm{\varphi_{l-1}(X)}_2}  + V_l \left( \varphi_{l-1}(X) -  \frac{\sqrt {k_{l-1}} \, \varphi_{l-1}(X)}{\norm{\varphi_{l-1}(X)}_2} \right) \right),
\end{equation}
which gives
\begin{equation}\label{eq:firstlemmaind}
    \norm{ \varphi_{l}(X) - \phi \left( V_l \frac{\sqrt {k_{l-1}} \, \varphi_{l-1}(X)}{\norm{\varphi_{l-1}(X)}_2} \right) }_2 \leq M C_1 \left| \norm{\varphi_{l - 1}(X)}_2 -\sqrt {k_{l-1}} \right| \leq M C_1 e^{C(l - 1)} \log D,
\end{equation}
where the first step is true since $\phi$ is $M$-Lipschitz, and the last step is a direct consequence of the inductive hypothesis \eqref{eq:indhyp}.

Let's now consider the second term in the left hand side of \eqref{eq:firstlemmaind}. Let's define the shorthand $\rho = V_l \frac{\sqrt {k_{l-1}} \, \varphi_{l-1}(X)}{\norm{\varphi_{l-1}(X)}_2} \in \R^D$. In the probability space of $V_l$, $\rho$ is distributed as a Gaussian random vector, such that all its entries $\rho_i$ are i.i.d. Gaussian with variance $\beta$. Thus, we have
\begin{equation}\label{eq:expnormind}
    \E_{V_l} \left [ \norm{\phi \left( V_l \frac{\sqrt {k_{l-1}} \, \varphi_{l-1}(X)}{\norm{\varphi_{l-1}(X)}_2} \right) }_2^2 \right ] = \E_{\rho} \left [ \norm{ \phi \left( \rho \right) }_2^2 \right ] = k_l \E_{\rho_1} \left [ \phi^2\left( \rho_1 \right) \right ] = k,
\end{equation}
where the last step is a consequence of \eqref{eq:heapp}, and of $k_{l} = k$ for $l \in [L]$.

As the $\rho_i$'s are independent and $\phi$ is Lipschitz, the random variables $\left( \phi^2\left( \rho_i \right) - 1 \right)$ are independent, mean-0, and sub-exponential, such that $\subEnorm{\phi^2\left( \rho_i \right) - 1 } \leq C_2$. Thus, by Bernstein inequality (cf. Theorem 2.8.1. in \cite{vershynin2018high}), we have
\begin{equation}\label{eq:Bernind}
    \P_{\rho} \left( \left| \sum_{i = 1}^k \left(\phi^2\left( \rho_i \right) - 1 \right) \right| \geq \sqrt k \log k \right) \leq 2 \exp (-c_2 \log^2 k),
\end{equation}
which gives
\begin{equation}
    \left | \norm{\phi \left( V_l \frac{\sqrt k \, \varphi_{l-1}(X)}{\norm{\varphi_{l-1}(X)}_2} \right) }_2^2 - k \right| \leq \sqrt k \log k,
\end{equation}
with probability at least $1 - 2 \exp (-c_2 \log^2 k)$ over $V_l$. We will condition on such high probability event until the end of the proof.

Thus, we have
\begin{equation}
    \norm{\phi \left( V_l \frac{\sqrt k \, \varphi_{l-1}(X)}{\norm{\varphi_{l-1}(X)}_2} \right) }_2 \leq  \sqrt{k  +  \sqrt k \log k} = \sqrt k \sqrt{1 + \frac{\log k}{\sqrt k}} \leq \sqrt k \left( 1 + \frac{\log k}{\sqrt k} \right) = \sqrt k + \log k,
\end{equation}
and
\begin{equation}
    \norm{\phi \left( V_l \frac{\sqrt k \, \varphi_{l-1}(X)}{\norm{\varphi_{l-1}(X)}_2} \right) }_2 \geq  \sqrt{k  -  \sqrt k \log k} = \sqrt k \sqrt{1 - \frac{\log k}{\sqrt k}} \geq \sqrt k \left( 1 - \frac{\log k}{\sqrt k} \right) = \sqrt k - \log k.
\end{equation}

Putting together the last two equations gives
\begin{equation}\label{eq:lemmaindtwo}
    \left | \norm{\phi \left( V_l \frac{\sqrt k \, \varphi_{l-1}(X)}{\norm{\varphi_{l-1}(X)}_2} \right) }_2 - \sqrt k \right| \leq  \log k.
\end{equation}

Applying \eqref{eq:firstlemmaind}, \eqref{eq:lemmaindtwo} and the triangle inequality gives
\begin{equation}
    \left | \norm{\varphi_l \left( X \right) }_2 - \sqrt k \right| \leq MC_1  e^{C(l-1)} \log k + \log k \leq \left( MC_1 e^{C(l-1)} + e^{C(l-1)} \right) \log k \leq e^{Cl} \log k,
\end{equation}
where the second and the third step are both consequences of \eqref{eq:Cind}. This inequality, performing a union bound on the high probability events we considered so far, holds with probability at least $1 - 2 (l - 1) \exp (-c \log^2 k) - 2 \exp (-c_1 k) - 2 \exp (-c_2 \log^2 k) \geq 1 - 2 l \exp (-c \log^2 k)$ over $\{ V_m \}_{m = 1}^l$, as soon as we consider $c = \min(c_1, c_2)$.
\end{proof}

\begin{lemma}\label{lemma:0drf}
    Let $\varphi_l(X)$ be defined in \eqref{eq:drffeatures}, let $\phi$ be a Lipschitz function and $X \in \R^{n \times d}$ a generic input sample such that Assumption \ref{ass:d} holds. Assume $k = \Theta(D)$. Then, for every $l \geq 0$, for every $i \in [n]$ and for any $\Delta \in \R^d$ such that $\norm{\Delta}_2 \leq \sqrt d$, we have
    \begin{equation}
        \norm{\varphi_{l}(X^i(\Delta)) - \varphi_{l}(X)}_2 \leq \sqrt d e^{Cl},
    \end{equation}
    with probability at least $1 - 2 L \exp(-c k)$ over $\{ V_l \}_{l = 1}^L$. 
\end{lemma}
\begin{proof}
    Let's prove the statement by induction over $l$. The base case $l=0$ is a direct consequence of $\norm{\Flat(X^i(\Delta)) - \Flat(X)}_2 = \norm{\Delta}_2 \leq \sqrt d$, which makes the thesis true for any $C \geq 0$.
    
    In the inductive step, the inductive hypothesis becomes
    \begin{equation}\label{eq:indhyplemma2}
        \norm{\varphi_{l-1}(X^i(\Delta)) - \varphi_{l-1}(X)}_2 \leq \sqrt d e^{C(l-1)},
    \end{equation}
    with probability at least $1 - 2 (l - 1) \exp (-c k)$ on $\{ V_m \}_{m = 1}^{l-1}$. Let's condition on this high probability event and on $\opnorm{V_l} \leq C_1$ until the end of the proof. By Theorem 4.4.5 of \cite{vershynin2018high} and since $k = \Theta(D)$, there exists $C_1$ large enough and independent from $l$, such that this holds with probability at least $1 - 2 \exp (-c_1 k)$ over $V_l$.
    
    We aim to prove the thesis for $e^C = M C_1$, and $c = c_1$, where $M$ is the Lipschitz constant of $\phi$. We have 
    \begin{equation}
    \begin{aligned}
        \norm{\varphi_{l}(X^i(\Delta)) - \varphi_{l}(X)}_2 &= \norm{\phi(V_l \varphi_{l-1}(X^i(\Delta)) ) - \phi (V_l \varphi_{l-1}(X) )}_2 \\
        &\leq M C_1 \norm{\varphi_{l-1}(X^i(\Delta)) - \varphi_{l-1}(X)}_2 \\
        &\leq M C_1 \sqrt d  e^{C(l-1)} \\
        &\leq \sqrt d e^{Cl},
    \end{aligned}
    \end{equation}
    with probability at least $1 - 2 (l - 1) \exp (-c k) - 2 \exp (-c_1 k) = 1 - 2 l \exp (-c k)$ over $\{ V_m \}_{m = 1}^l$, which gives the thesis.
\end{proof}

\paragraph{Theorem \ref{thm:drf}.}
Let $\varphi_{\textup{DRF}}(X)$ be the deep random feature map defined in \eqref{eq:drf}, and let $\phi$ be a Lipschitz function, such that \eqref{eq:He} admits at least one solution $\beta$. Let $X \in \R^{n \times d}$ be a generic input sample s.t.\ Assumption \ref{ass:d} holds, and assume $k = \Theta(D)$, $L = o (\log k)$. Let $\mathcal S_{\textup{DRF}}(X)$ be the the word sensitivity defined in \eqref{eq:sensitivity}. Then, we have
\begin{equation}
    \mathcal S_{\textup{DRF}} = \bigO{\frac{e^{CL}}{\sqrt n}},
\end{equation}
with probability at least $1 - \exp(-c \log^2 k)$ over $\{ V_m \}_{m = 1}^l$. 
\begin{proof}
    By Lemma \ref{lemma:concentrationnorminduction} (for the case $l = L$) and since $L = o (\log k)$, we have
    \begin{equation}\label{eq:thmdrf1}
        \norm{\varphi_{\textup{DRF}}(X)}_2 = \Theta(\sqrt k),
    \end{equation}
    with probability at least $1 - 2 L \exp (-c_1 \log^2 k) \geq 1 - \exp(-c_2 \log^2 k)$ over $\{ V_l \}_{l = 1}^L$.

    By Lemma \ref{lemma:0drf} (for the case $l = L$), we have that, for every $i \in [n]$, we have
    \begin{equation}\label{eq:thmdrf2}
        \sup_{\norm{\Delta}_2 \leq \sqrt d} \norm{\varphi_{\textup{DRF}}(X^i(\Delta)) - \varphi_{\textup{DRF}}(X)}_2 \leq \sqrt d e^{CL},
    \end{equation}
    with probability at least $1 - 2 L \exp (-c_3 k) \geq 1 - \exp(-c_4 \log^2 k)$ over $\{ V_l \}_{l = 1}^L$.

    Then, putting together \eqref{eq:thmdrf1}, \eqref{eq:thmdrf2}, and \eqref{eq:sensitivity}, and recalling that $k = \Theta(D)$, we get the thesis.
\end{proof}


\section{Proofs for random attention features}\label{app:raf}

In this section, we provide the proofs for our results on the random attention features model.

We consider a single textual data-point $X = [x_1, x_2, ..., x_n]^\top \in \R^{n \times d}$ that satisfies Assumptions \ref{ass:d}. We assume $d / \log^4 d = \Omega(n)$. We consider the random attention features model defined in \eqref{eq:raf}, i.e.,
\begin{equation}
    \varphi_{\textup{RAF}}(X) = \softmax \left( \frac{X W X^\top}{\sqrt{d}} \right) X,
\end{equation}
where $W_{i,j} \distas{}_{\rm i.i.d.}\mathcal{N}(0, 1 / d)$ sampled independently from $X$. We define the argument of the $\softmax$ as $S(X) \in \R^{n \times n}$ given by
\begin{equation}\label{eq:SX}
    S(X) = \frac{X W X^\top}{\sqrt d},
\end{equation}
and its evaluation on the perturbed sample as
\begin{equation}\label{eq:SXD}
    S(X^i(\Delta)) = \frac{X^i(\Delta) W X^i(\Delta)^\top}{\sqrt d}.
\end{equation}
The attention scores are therefore defined as the row-wise $\softmax$ of the previous term, i.e.,
\begin{equation}\label{eq:fromStos}
    [s(X)]_{j:} = \softmax \left(  [S(X)]_{j:} \right),
\end{equation}
which allows us to rewrite \eqref{eq:raf} as
\begin{equation}\label{eq:rafissx}
    \varphi_{\textup{RAF}}(X) = s(X) X.
\end{equation}

Through this section, we will always consider this notation and assumptions, which will therefore not be repeated in the statement of the lemmas before our final result Theorem \ref{thm:RAF}.

The outline of this section is the following:
\begin{enumerate}
    \item In Lemma \ref{lemma:dstar}, we prove the existence of a vector $\delta^* \in \R^d$, such that its inner product with the token embeddings $x_i$'s is large for a constant fraction of the words in the context.
    \item In Lemma \ref{lemma:dstar12}, we show that the result in the previous Lemma can be extended to two different vectors $\delta^*_1$ and $\delta^*_2$, that are very different from each other, i.e., $\norm{\delta^*_1 - \delta^*_2}_2 = \sqrt d$. The reason why we would like this statement to hold on two vectors will be clear in Lemma \ref{lemma:numeratoromega}.
    \item In Lemma \ref{lemma:Dstar}, exploiting the properties of the Gaussian attention features $W$, we show that there exist two different vectors $\Delta^*_1$ and $\Delta^*_2$, that are very different from each other, i.e., $\norm{\Delta^*_1 - \Delta^*_2}_2 = \Omega(\sqrt d)$, such that a constant fraction of the entries of the vector $X W \Delta^*_k /\sqrt d$, $k\in \{1, 2\}$, is large.
    \item In Lemma \ref{lemma:scoresstar}, we prove that, for the two vectors $\Delta^*_1$ and $\Delta^*_2$ and an $\Omega(n)$ number of rows $j \in [n]$, the attention scores $[s(X^i(\Delta^*_{k}))]_{j:}^\top$
    are well approximated by the canonical basis vector $e_i$ for $k\in \{1, 2\}$. This intuitively means that a constant fraction of tokens $x_j$ puts all their attention towards the $i$-th modified token $x_i + \Delta^*_{k}$.
    \item In Lemma \ref{lemma:numeratoromega}, using an argument by contradiction, we prove that at least one between $\Delta^*_1$ and $\Delta^*_2$ is such that $\norm{\varphi_{\textup{RAF}}(X) - \varphi_{\textup{RAF}}(X^i(\Delta^*))}_F = \Omega(\sqrt {dn})$.
    \item In Theorem \ref{thm:RAF}, we upper bound $\norm{\varphi_{\textup{RAF}}(X)}_F$, concluding the argument.  
\end{enumerate}

\begin{lemma}\label{lemma:dstar}
    There exists a vector $\delta^* \in \R^d$, such that 
    \begin{equation}
        \norm{\delta^*}_2 \leq \sqrt d,
    \end{equation}
    and
    \begin{equation}
        \left( x_i^\top \delta^* \right)^2 = \Omega\left( \frac{d^2}{n} \right),
    \end{equation}
    for at least $\Omega(n)$ indices $i \in [n]$.
\end{lemma}
\begin{proof}
    Let $X^\top = U D V^\top$ the singular value decomposition of $X^\top$, where $U \in \R^{d \times d}$, $D \in \R^{d \times n}$, and $V \in \R^{n \times n}$, and let $U_r \in \R^{d \times r}$ be the matrix obtained keeping only the first $r = \Rank(X)$ columns of $U$.
    Let $\delta_\varepsilon$ be defined as follows
    \begin{equation}
        \delta_\varepsilon := \sqrt{\frac{d}{n}} U_r \varepsilon,
    \end{equation}
    where we consider $\varepsilon \in \R^r$ uniformly distributed on the sphere $\sqrt r \, \mathbb S^{r - 1}$ of radius $\sqrt r$.
    By definition we have $U_r^\top U_r = I$, which implies
    \begin{equation}
        \norm{U_r \varepsilon}_2 = \norm{\varepsilon}_2 = \sqrt{r}.
    \end{equation}
    As $r \leq n$, we have that
    \begin{equation}\label{eq:normdeltaeps}
        \norm{\delta_\varepsilon}_2 \leq \sqrt{d},
    \end{equation}
    for every $\varepsilon$.
    Let's call $Z^\varepsilon_i$ the random variable $x_i^\top \delta_\varepsilon$. 
    
    By contradiction, let's assume the thesis to be false. Then, in particular, by \eqref{eq:normdeltaeps}, there is no $\varepsilon$ such that $\delta_\varepsilon$ respects the second part of the thesis. This implies that, for every $\varepsilon$, there are at least $ \left\lceil n / 2 \right\rceil $ values of $i$ such that $\left( Z^\varepsilon_i \right)^2 < 0.01 \, d^2 / n$. Let's define the indicators
    \begin{equation}\label{eq:indicators}
        \chi_i = \begin{cases}
        1, & \text{if } \left( Z_i^\varepsilon \right)^2 < 0.01 \, d^2 / n, \\
        0, & \text{if } \left( Z_i^\varepsilon \right)^2 \geq 0.01 \, d^2 / n.
    \end{cases}
    \end{equation}
    Thus, by contradiction, we have
    \begin{equation}
        \sum \chi_i \geq \left\lceil \frac{n}{2} \right\rceil \geq \frac{n}{2},
    \end{equation}
    for every $\varepsilon$. Thus, by the probabilistic method,
    \begin{equation}
        \frac{n}{2} \leq \E_{\varepsilon} \left[ \sum  \chi_i \right] = \sum \E_{\varepsilon} \left[  \chi_i \right] = n \, \E_{\varepsilon} \left[  \chi_1 \right] = n \, \P_\varepsilon\left( \left( Z^\varepsilon_1 \right)^2 \leq 0.01\, d^2 / n \right),
    \end{equation}
    where the second equality is true as the $Z^\varepsilon_i$ are identically distributed, and the last step comes from the definition of the indicators in \eqref{eq:indicators}. This implies, for every $i$,
    \begin{equation}\label{eq:contradiction1}
        \P_\varepsilon\left( \left( Z^\varepsilon_i \right)^2 \leq 0.01\, d^2 / n \right) \geq 0.5.
    \end{equation}
    
    Let $\rho \in \R^r$ be a standard Gaussian vector 
    and 
    define $\rho_r = \norm{\rho}_2 / \sqrt{r}$. We have that $\E[\rho_r^2] = 1$ and $\textup{Var}(\rho_r^2) = 1 / r$. Thus, for every $r$, by Chebyshev's inequality, we can write
    \begin{equation}\label{eq:contradiction2}
        \P_{\rho_r}\left( \rho_r^2 \leq 3 \right) \geq 0.75 \geq 0.5.
    \end{equation}

    By rotational invariance of the Gaussian measure, we have that $\rho_r \varepsilon$ is distributed as a standard Gaussian vector. In particular, we have
    \begin{equation}
        \rho_r Z_i^\varepsilon = \rho_r x_i^\top \delta_\varepsilon = \left( \sqrt{\frac{d}{n}} U_r^\top x_i\right)^\top \left( \rho_r \varepsilon \right).
    \end{equation}
    As $U_r^\top x_i$ is a fixed vector independent from $\varepsilon$ and $\rho_r$, we have that $\rho_r Z_i^\varepsilon$ is a Gaussian random variable (in the probability space of $\varepsilon$ and $\rho_r$), with variance $d \norm{U_r^\top x_i}_2^2 / n = d^2 / n$, as $x_i \in \Span \{ \Rows (U_r^\top) \}$ by construction.
    Thus, $\sqrt{n} \rho_r Z^\varepsilon_i / d$ is distributed as a standard Gaussian random variable $g_i$.
    
    Therefore, putting together \eqref{eq:contradiction1} and \eqref{eq:contradiction2} gives
    \begin{equation}
        \P_{g_i} \left(  g_i^2 \leq 0.03 \right) \geq 0.25.
    \end{equation}
    However, we can upper-bound the left hand side of the previous equation exploiting the closed form of the pdf of a standard Gaussian random variable. This gives
    \begin{equation}
        \P_{g_i} \left(  |g_i| \leq \sqrt{0.03} \right) < \frac{1}{\sqrt{2 \pi}} \cdot 2 \cdot \sqrt{0.03} < 0.14 < 0.25,
    \end{equation}
    which leads to the desired contradiction.
\end{proof}

\begin{lemma}\label{lemma:dstar12}
    There exist two vectors $\delta^*_1 \in \R^d$ and $\delta^*_2 \in \R^d$ such that 
    \begin{equation}
        \norm{\delta^*_1}_2 \leq \sqrt d, \qquad \norm{\delta^*_2}_2 \leq \sqrt d, \qquad \norm{\delta^*_1 - \delta^*_2}_2 = \sqrt d, 
    \end{equation}
    and $x_i^\top \delta^*_1 > 0$, $x_i^\top \delta^*_2 > 0$, and
    \begin{equation}
        x_i^\top \delta^*_1 = \Omega\left( \sqrt d \log^2 d \right),
    \end{equation}
    \begin{equation}
        x_i^\top \delta^*_2= \Omega\left( \sqrt d \log^2d \right),
    \end{equation}
    for at least $\Omega(n)$ indices $i \in [n]$.
\end{lemma}
\begin{proof}
    By Lemma \ref{lemma:dstar}, we have that there exists a vector $\delta^*$, with $\norm{\delta^*}_2 \leq \sqrt d$ such that
    \begin{equation}
        \left( x_i^\top \delta^* \right)^2 = \Omega\left( \frac{d^2}{n} \right),
    \end{equation}
    for at least $\Omega(n)$ indices $i \in [n]$. This also implies (up to considering $-\delta^*$ instead of $\delta^*$) that 
    \begin{equation}
        x_i^\top \delta^* = \Omega\left( \frac{d}{\sqrt n} \right) = \Omega \left( \sqrt d \log^2 d\right),
    \end{equation}
    where the last step is justified by $d / \log^4 d = \Omega(n)$.
    
    Let's now define
    \begin{equation}
        \delta^*_1 = \frac{\delta^*}{2} + \frac{v}{2}, 
    \end{equation}
    and
    \begin{equation}
        \delta^*_2 = \frac{\delta^*}{2} - \frac{v}{2}, 
    \end{equation}
    where $v \in \R^d$ is a generic fixed vector such that $\norm{v}_2 = \sqrt d$, and $\norm{X v}_2 = 0$, i.e. $v \in \Span \{ \Rows (X) \}  ^\perp$. This is again possible because of $d / \log^4 d = \Omega(n)$. As $\norm{\delta^*}_2 \leq \sqrt{d}$ by Lemma \ref{lemma:dstar}, the first part of the thesis follows from a straightforward application of the triangle inequality.

    Since $v^\top x_i = 0$ for every $1 \leq i \leq n$, we also have
    \begin{equation}
        x_i^\top \delta^*_1 = x_i^\top \delta^*_2 = x_i^\top \delta^* / 2,
    \end{equation}
    which readily gives the desired result.
\end{proof}

\begin{lemma}\label{lemma:Dstar}
    Let's define $\sigma(\Delta) \in \R^n$ as
    \begin{equation}
        \sigma(\Delta) := \frac{X W \Delta}{\sqrt d},
    \end{equation}
    where $\Delta \in \R^d$. Then, with probability at least $1 - \exp(-c \log^2 d)$ over $W$, there exist $\Delta^*_1$ and $\Delta^*_2$ such that,
    \begin{equation}
        \norm{\Delta^*_1}_2 \leq \sqrt d, \qquad \norm{\Delta^*_2}_2 \leq \sqrt d, \qquad \norm{\Delta^*_1 - \Delta^*_2}_2 = \Omega(\sqrt{d}), 
    \end{equation}
    and 
    \begin{equation}
        \sigma(\Delta^*_1)_i = \Omega \left( \log^2 d \right).
    \end{equation}
    \begin{equation}
        \sigma(\Delta^*_2)_i = \Omega \left( \log^2 d \right).
    \end{equation}
    for at least $\Omega(n)$ indices $i \in [n]$.
\end{lemma}
\begin{proof}
    Let's set
    \begin{equation}
        \Delta = C W^\top \delta,
    \end{equation}
    where $\delta \in \R^d$ is independent from $W$, and $C$ is an absolute constant that will be fixed later in the proof. For every $i$, we can write
    \begin{equation}
        \delta = \frac{\delta^\top x_i}{\norm{x_i}_2^2} x_i + \delta_i^\perp,
    \end{equation}
    with $\delta_i^\perp$ being orthogonal with $x_i$ by construction. Let's for now suppose $\norm{\delta_i^\perp}_2 \neq 0$. Let $a(\delta)_i := \delta^\top x_i$.
    We have
    \begin{equation}
    \begin{aligned}
        \sigma(\Delta)_i &\geq C \frac{a(\delta)_i}{ \sqrt d \norm{x_i}_2^2}\norm{W^\top x_i}_2^2 - C \left| \frac{x_i^\top W W^\top \delta_i^\perp}{\sqrt d} \right| \\
        & = C  \frac{a(\delta)_i}{ \sqrt d \norm{x_i}_2^2} \norm{W^\top x_i}_2^2 - C \frac{\norm{\delta_i^\perp}_2}{d} \left| x_i^\top W W^\top \frac{\sqrt d \,\delta_i^\perp}{\norm{\delta_i^\perp}_2} \right|.
    \end{aligned}
    \end{equation}
    In the probability space of $W$, as the entries of $W$ are i.i.d. and Gaussian $\mathcal N (0, 1/d)$, we have that $W^\top x_i$ and $W^\top \frac{\sqrt d \,\delta_i^\perp}{\norm{\delta_i^\perp}_2}$ are two independent standard Gaussian vectors, namely, $\rho_1$ and $\rho_2$. This implies
    \begin{equation}\label{eq:highpDstar}
        \norm{W^\top x_i}_2^2 = \Omega(d), \qquad \left |x_i^\top W W^\top \frac{\sqrt d \,\delta_i^\perp}{\norm{\delta_i^\perp}_2} \right| =: \left| \rho_1^\top \rho_2 \right| = \bigO{\sqrt d \log d},
    \end{equation}
    with probability at least $1 - \exp (- c_1 \log^2 d)$, over the probability space of $W$. The first statement holds because of Theorem 3.1.1. of \cite{vershynin2018high}, and the second one because, in the probability space of $\rho_1$, $\rho_1^\top \rho_2$ is a Gaussian random variable with variance $\norm{\rho_2}_2^2$, which in turn is $\bigO{d}$ with probability at least $1 - \exp (- c_2 d)$ over $\rho_2$. Using $\norm{x_i}_2^2 = d$, 
    we get, for two positive absolute constants $C_1$ and $C_2$,
    \begin{equation}
        \sigma(\Delta)_i \geq C_1 \frac{ a(\delta)_i }{\sqrt d} - C_2 \frac{\norm{\delta^\perp_i}_2 \log d}{\sqrt d} \geq C_1 \frac{ a(\delta)_i }{\sqrt d} - C_2 \frac{\norm{\delta}_2 \log d}{\sqrt d}.
    \end{equation}
    Notice that the previous equation would still hold even in the case $\norm{\delta_i^\perp}_2 = 0$, which is therefore a case now included in our derivation.
    
    By Lemma \ref{lemma:dstar12}, we can choose two vectors $\delta^*_1$ and $\delta^*_2$, independently from $W$ such that, for $k\in\{1, 2\}$, $\norm{\delta^*_{k}}_2 \leq \sqrt d$ and $a(\delta^*_{k})_i  = \Omega(\sqrt d \log^2 d)$ for $\Omega(n)$ indices $i \in [n]$. This gives, for another absolute positive constant $C_3$,  
    \begin{equation}
        \sigma(\Delta^*_1)_i \geq C_3 \log^2 d - C_2 \log d = \Omega (\log^2 d),
    \end{equation}
    \begin{equation}
        \sigma(\Delta^*_2)_i \geq C_3 \log^2 d - C_2 \log d  = \Omega(\log^2 d),
    \end{equation}
    for $\Omega(n)$ indices $i \in [n]$.

    Let's now verify that $\norm{\Delta^*_1 - \Delta^*_2}_2 = C \norm{W^\top (\delta^*_1 - \delta^*_2)}_2 = \Omega(\sqrt{d})$. As $\delta^*_1 - \delta^*_2$ is independent on $W$, and $\norm{\delta^*_1 - \delta^*_2}_2 = \sqrt d$, we have that $W^\top (\delta^*_1 - \delta^*_2)$ is a standard Gaussian random vector, in the probability space of $W$. Thus, by Theorem 3.3.1 of \cite{vershynin2018high}, we have
    \begin{equation}\label{eq:bignormDstar12}
        \norm{\Delta^*_1 - \Delta^*_2}_2 = C \norm{W^\top (\delta^*_1 - \delta^*_2)}_2 = \Omega(\sqrt{d}),
    \end{equation}
    with probability at least $1 - \exp(-c_3 d)$, on the probability space of $W$.
    
    We are left to verify that, for $k\in \{1, 2\}$, $\norm{\Delta^*_{k}}_2 = C \norm{W^\top \delta_k^*}_2 \leq \sqrt d$. This is readily implied by Theorem 4.4.5 of \cite{vershynin2018high}, which gives $\opnorm{W} = \bigO{1}$ with probability at least $1 - \exp (-c_4 d)$. Taking the intersection between this high probability event and the ones in \eqref{eq:highpDstar} and \eqref{eq:bignormDstar12}, and setting $C$ small enough to have $C \opnorm{W} \leq 1$, we get the thesis.
\end{proof}

\begin{lemma}\label{lemma:scoresstar}
    For every $i \in [n]$, with probability at least $1 - \exp(-c \log^2 d)$ over $W$, there exist $\Delta^*_1$ and $\Delta^*_2$ such that,
    \begin{equation}
        \norm{\Delta^*_1}_2 \leq \sqrt d, \qquad \norm{\Delta^*_2}_2 \leq \sqrt d, \qquad \norm{\Delta^*_1 - \Delta^*_2}_2 = \Omega(\sqrt{d}), 
    \end{equation}
    and
    \begin{equation}
        \norm{[s(X^i(\Delta^*_1))]_{j:}^\top - e_i}_2 = o\left( \frac{1}{\sqrt d}\right),
    \end{equation}
    \begin{equation}
        \norm{[s(X^i(\Delta^*_2))]_{j:}^\top - e_i}_2 = o\left( \frac{1}{\sqrt d}\right),
    \end{equation}
    for at least $\Omega(n)$ indices $j \in [n]$, where $e_i$ represents the $i$-th element of the canonical basis in $\R^n$.
\end{lemma}
\begin{proof}
    As we can write $X^i(\Delta) = X + e_i \Delta^\top$, we can rewrite \eqref{eq:SXD} as
    \begin{equation}
        \sqrt d \, S(X^i(\Delta)) = X^i(\Delta) W X^i(\Delta)^\top = X W X^\top + X W \Delta e_i^\top + e_i \Delta^\top W X + e_i \Delta^\top W \Delta e_i^\top.
    \end{equation}
    Let's look at the $j$-th row of $S(X^i(\Delta))$, with $j \neq i$. We have
    \begin{equation}\label{eq:presoftrow}
        [S(X^i(\Delta))]_{j:} = \frac{x_j^\top W X^\top}{\sqrt d} + \frac{x_j^\top W \Delta e_i^\top}{\sqrt d}.
    \end{equation}
    With probability at least $1 - \exp(-c_1 \log^2 d)$ over $W$, we have that all the $n$ entries of the vector $x_j^\top W X^\top / \sqrt d$ are smaller than $\log d$, as they are all standard Gaussian random variables. By Lemma \ref{lemma:Dstar}, we have that there exist $\Delta_1^*$ and $\Delta_2^*$ such that $\norm{\Delta^*_1}_2 \leq \sqrt d$, $\norm{\Delta^*_2}_2 \leq \sqrt d$, $\norm{\Delta^*_1 - \Delta^*_2}_2 = \Omega(\sqrt{d})$ and 
    \begin{equation}\label{eq:bdbreak}
        \frac{x_j^\top W \Delta^*_{k}}{\sqrt d} = \Omega (\log^2 d),\qquad k\in \{1, 2\},
    \end{equation}
    for $\Omega(n)$ indices $j \in [n]$, with probability at least $1 - \exp(-c_2 \log^2 d)$ over $W$.
    Let's consider this set of indices until the end of the proof, where we have that, for $k\in\{1, 2\}$, $[S(X^i(\Delta^*_{k}))]_{ji} = \Omega (\log^2 d)$.
    
    After applying the $\softmax$ function to the row in \eqref{eq:presoftrow} as indicated in \eqref{eq:fromStos}, we get, for every index $k \in [n]$ and $k \neq i$
    \begin{equation}
    \begin{aligned}
        [s(X^i(\Delta^*_1))]_{jk} &= \softmax \left(  [S(X^i(\Delta^*_1))]_{j:} \right)_k \\
        &= \frac{\exp \left( [S(X^i(\Delta^*_1))]_{jk} \right) }{\sum_l \exp \left( [S(X^i(\Delta^*_1))]_{jl} \right) } \\
        &\leq   \frac{ d }{\exp \left( [S(X^i(\Delta^*_1))]_{ji}  \right)} \\
        &\leq   \frac{ d }{C d ^{\log d}} = o \left( \frac{1}{d} \right),
    \end{aligned}
    \end{equation}
    and
    \begin{equation}
    \begin{aligned}
        [s(X^i(\Delta^*_1))]_{ji} &= \softmax \left(  [S(X^i(\Delta^*_1))]_{j:} \right)_i \\
        & = 1 - \frac{\sum_{l\neq i} \exp \left( [S(X^i(\Delta^*_1))]_{jl} \right) }{\sum_l \exp \left( [S(X^i(\Delta^*_1))]_{jl} \right) } \\
        & \geq 1 - \frac{ nd }{\exp \left( [S(X^i(\Delta^*_1))]_{ji}  \right)} \\
        & \geq 1 -  \frac{ nd }{C d ^{\log d}} = 1 - o \left( \frac{1}{d} \right),
    \end{aligned}
    \end{equation}
    where the last step is justified by the fact that $d / \log^4 d = \Omega(n)$.

    Thus, for the same reason, we have
    \begin{equation}\label{eq:smallnormscoreres}
        \norm{[s(X^i(\Delta^*_1))]_{j:}^\top - e_i}_2^2 = \sum_{k \neq i} [s(X^i(\Delta^*_1))]_{jk}^2 + \left(1 - [s(X^i(\Delta^*_1))]_{ji}\right) ^2 = n o \left( \frac{1} {d^2}\right) = o \left( \frac{1} {d}\right).
    \end{equation}

    As \eqref{eq:smallnormscoreres} can be written also with respect to $\Delta^*_2$, the thesis readily follows. 
\end{proof}

\begin{lemma}\label{lemma:numeratoromega}

For every $i \in [n]$, with probability at least $1 - \exp(-c \log^2 d)$ over $W$, there exist $\Delta^*$ such that $\norm{\Delta^*}_2 \leq \sqrt d$, and
\begin{equation}
    \norm{\varphi_{\textup{RAF}}(X) - \varphi_{\textup{RAF}}(X^i(\Delta^*))}_F = \Omega(\sqrt {dn}).
\end{equation}

\end{lemma}

\begin{proof}
By \eqref{eq:rafissx}, we have that,
\begin{equation}
    \varphi_{\textup{RAF}}(X^i(\Delta)) = s(X^i(\Delta)) X^i(\Delta),
\end{equation}
By Lemma \ref{lemma:scoresstar}, we have that there exists $\Delta^*_1$, with $\norm{\Delta^*_1}_2 \leq \sqrt d$, such that there are $\Omega(n)$ indices $j$ such that 
\begin{equation}
    \norm{[s(X^i(\Delta^*_1))]_{j:}^\top - e_i}_2 = o\left( \frac{1}{\sqrt d}\right),
\end{equation}
with probability at least $1 - \exp(-c_1 \log^2 d)$ over $W$. Until the end of the proof, we will refer to $j$ as a generic index for which the previous equation holds, and we will condition on the high probability event that the equations in the statement of Lemma \ref{lemma:scoresstar} hold.

As we can write 
\begin{equation}
    \left[\varphi_{\textup{RAF}}(X^i(\Delta^*_1))\right]_{j:} = [s(X^i(\Delta^*_1))]_{j:} X^i(\Delta^*_1) = e_i^\top X^i(\Delta^*_1) + \left([s(X^i(\Delta^*_1))]_{j:} - e_i^\top \right) X^i(\Delta^*_1),
\end{equation}
and $e_i^\top X^i(\Delta^*_1) = x_i^\top + {\Delta^*_1}^\top$, we have
\begin{equation}\label{eq:rowsRAF}
\begin{aligned}
    \norm{ \left[\varphi_{\textup{RAF}}(X^i(\Delta^*_1))\right]_{j:} -  \left( x_i + \Delta^*_1 \right)^\top}_2 &\leq  \norm{[s(X^i(\Delta^*_1))]_{j:} - e_i^\top}_2 \opnorm{X^i(\Delta^*_1)} \\
    &\leq o \left( \frac{1}{\sqrt d} \right) \left( \norm{X}_F + \norm{\Delta^*_1}_2 \right) \\
    & = o \left( \frac{1}{\sqrt d} \right) \bigO{\sqrt{dn}} = o\left(\sqrt d\right),
\end{aligned}
\end{equation}
where the last step follows from our assumption $d / \log^4 d = \Omega(n)$. By Lemma \ref{lemma:scoresstar}, there also exists $\Delta^*_2$ such that $\norm{\Delta^*_2}_2 \leq \sqrt d$, $\norm{\Delta^*_1 - \Delta^*_2}_2 = \Omega(\sqrt{d})$ and \eqref{eq:rowsRAF} holds.

Let's now suppose by contradiction that, for some $j$,
\begin{equation}\label{eq:contradictionD12}
    \norm{\left[\varphi_{\textup{RAF}}(X)\right]_{j:} - \left[\varphi_{\textup{RAF}}(X^i(\Delta^*_1))\right]_{j:}}_2 = o \left( \sqrt d \right), \qquad \norm{\left[\varphi_{\textup{RAF}}(X)\right]_{j:} - \left[\varphi_{\textup{RAF}}(X^i(\Delta^*_2))\right]_{j:}}_2 = o \left( \sqrt d \right).
\end{equation}
Then, we have
\begin{equation}
    \begin{aligned}
        \norm{\Delta^*_1 - \Delta^*_2}_2 =& \norm{ \left(x_i + \Delta^*_1\right) - \left(x_i + \Delta^*_2\right)}_2\\
        =& \, \| \left( \left(x_i + \Delta^*_1\right) - \left[\varphi_{\textup{RAF}}(X^i(\Delta^*_1))\right]_{j:} \right) + \left( \left[\varphi_{\textup{RAF}}(X^i(\Delta^*_1))\right]_{j:} - \left[\varphi_{\textup{RAF}}(X)\right]_{j:} \right) \\
        & - \left( \left(x_i + \Delta^*_2\right) - \left[\varphi_{\textup{RAF}}(X^i(\Delta^*_2))\right]_{j:} \right) - \left( \left[\varphi_{\textup{RAF}}(X^i(\Delta^*_2))\right]_{j:} - \left[\varphi_{\textup{RAF}}(X)\right]_{j:} \right) \| \\
        \le & \norm{\left(x_i + \Delta^*_1\right) - \left[\varphi_{\textup{RAF}}(X^i(\Delta^*_1))\right]_{j:}}_2 + \norm{ \left[\varphi_{\textup{RAF}}(X^i(\Delta^*_1))\right]_{j:} - \left[\varphi_{\textup{RAF}}(X)\right]_{j:} }_2 \\
        &+ \norm{\left(x_i + \Delta^*_2\right) - \left[\varphi_{\textup{RAF}}(X^i(\Delta^*_2))\right]_{j:}}_2 + \norm{\left[\varphi_{\textup{RAF}}(X^i(\Delta^*_2))\right]_{j:} - \left[\varphi_{\textup{RAF}}(X)\right]_{j:}}_2 \\
        =& \, o(\sqrt d),
    \end{aligned}
\end{equation}
where the third step holds by triangle inequality, and the last comes from \eqref{eq:rowsRAF} and \eqref{eq:contradictionD12}. As $\norm{\Delta^*_1 - \Delta^*_2}_2 = \Omega(\sqrt d)$ by Lemma \ref{lemma:scoresstar}, we get the desired contradiction, and we have that at least one equation in \eqref{eq:contradictionD12} doesn't hold.

Let's therefore denote by $\Delta^*$ the vector with $\norm{\Delta^*}_2 \leq \sqrt d$ such that
\begin{equation}
\norm{\left[\varphi_{\textup{RAF}}(X)\right]_{j:} - \left[\varphi_{\textup{RAF}}(X^i(\Delta^*))\right]_{j:}}_2 = \Omega \left( \sqrt d \right)
\end{equation}
holds for $\Omega(n)$ indices $j$. Due to the previous conditioning, such a vector exists with probability at least $1 - \exp(-c_1 \log^2 d)$ over $W$. Let's denote the set of these indices by $J$, with $|J| = \Omega(n)$ (where we denote set cardinality by $| \cdot |$).

Thus, we have
\begin{equation}
    \norm{\varphi_{\textup{RAF}}(X) - \varphi_{\textup{RAF}}(X^i(\Delta^*))}_F^2 \geq \sum_{j \in J} \norm{\left[\varphi_{\textup{RAF}}(X)\right]_{j:} - \left[\varphi_{\textup{RAF}}(X^i(\Delta^*))\right]_{j:}}_2^2 = \Omega(nd),
\end{equation}
which concludes the proof.

\end{proof}


\paragraph{Theorem \ref{thm:RAF}} Let $\varphi_{\textup{RAF}}(X)$ be the random attention features map defined in \eqref{eq:raf}. Let $X \in \R^{n \times d}$ be a generic input sample s.t.\ Assumption \ref{ass:d} holds, and assume $ d / \log^4 d = \Omega(n)$. Let $\mathcal S_{\textup{RAF}}(X)$ be the the word sensitivity defined in \eqref{eq:sensitivity}. Then, we have
\begin{equation}
    \mathcal S_{{\textup{RAF}}}(X) = \Omega(1),
\end{equation}
with probability at least $1 - \exp(-c \log^2 d)$ over $W$.
\begin{proof}
    We have $\varphi_{\textup{RAF}}(X) = s(X) X$, which implies
    \begin{equation}
        \left[ \varphi_{\textup{RAF}}(X) \right]_{j:} =  \left[ s(X) \right]_{j:} X,
    \end{equation}
    and
    \begin{equation}
        \norm{\left[ \varphi_{\textup{RAF}}(X) \right]_{j:}}_2 = \norm{\sum_{k = 1}^n  \left[ s(X) \right]_{jk} x_k }_2 \leq  \sum_{k = 1}^n  \left[ s(X) \right]_{jk} \norm{ x_k }_2 = \sqrt d,
    \end{equation}
    where the second step follows from triangle inequality, and the last step holds as $\sum_{k = 1}^n  \left[ s(X) \right]_{jk} =1$, and $\norm{x_k}_2 = \sqrt d$ for every $k$.
    This readily implies
    \begin{equation}\label{eq:thmrafdenominator}
        \norm{\varphi_{\textup{RAF}}(X)}_F \leq \sqrt {nd}.
    \end{equation}

    By Lemma \ref{lemma:numeratoromega}, we have that, with probability at least $1 - \exp(-c \log^2 d)$ over $W$, there exist $\Delta^*$ such that $\norm{\Delta^*}_2 \leq \sqrt d$, and
    \begin{equation}
        \norm{\varphi_{\textup{RAF}}(X) - \varphi_{\textup{RAF}}(X^i(\Delta^*))}_F = \Omega(\sqrt {dn}).
    \end{equation}

    This, together with \eqref{eq:thmrafdenominator}, concludes the proof.
\end{proof}

\section{Assumption \ref{ass:uncertainity} and adversarial robustness}\label{app:robustness}

Assumption \ref{ass:uncertainity} requires the perturbation $\Delta$ to be such that the model $f_{\textup{RF}}(\cdot, \theta^*)$ gives a similar output when evaluated on the two new samples $X$ and $X^i(\Delta)$. This assumption is necessary to understand the different behaviour of the RF and the RAF model in our setting, as there in fact exists an adversarial patch $\Delta$ such that, for example, $f(X^i(\Delta), \theta^*_{r})$ and $f(X, \theta^*_{r})$ are very different from each other (e.g., $f_{\textup{RF}}(X^i(\Delta), \theta^*_{r}) = y_\Delta$ while $f_{\textup{RF}}(X, \theta^*_{r}) = y$). 

This conclusion derives from the adversarial vulnerability of the RF model, extensively studied in previous work \cite{dohmatob2022non, dohmatob, bombari2023universal}. We remark that this vulnerability depends on the scalings of the problem, i.e., $n, d$ and $N$. In fact, (using the re-trained solution as example) we can write (assuming $\theta_0 = 0$ for simplicity)
\begin{equation}
\begin{aligned}
    \left| f_{\textup{RF}}(X^i(\Delta), \theta^*_{r}) - f_{\textup{RF}}(X, \theta^*_{r}) \right| &= \left| \left( \varphi_{\textup{RF}}(X^i(\Delta)) - \varphi_{\textup{RF}}(X) \right) ^\top \Phi_{\textup{RF}, r}^+ \mathcal Y_r \right| \\
    &\leq \norm{\varphi_{\textup{RF}}(X^i(\Delta)) - \varphi_{\textup{RF}}(X) }_2 \opnorm{\Phi_{\textup{RF}, r}^+} \norm{Y_r}_2.
\end{aligned}
\end{equation}
If we bound the three terms on the RHS separately, we get:

\begin{itemize}
    \item $\norm{\varphi_{\textup{RF}}(X^i(\Delta)) - \varphi_{\textup{RF}}(X) }_2 = \bigO{\sqrt{k / n}}$ with probability at least $1 - \exp (-cD)$ over $V$, by Lemma \ref{lemma:0rf};

\item $\opnorm{\Phi_{\textup{RF}, r}^+} = \lambda_{\min}^{-1/2} (K_{\textup{RF}, r}) = \bigO{\sqrt{1 / k}}$, with probability at least $1 - \exp \left( -c \log^2 N \right)$ over $V$ and $\mathcal X_r$, by Lemma \ref{lemma:evminRF};

\item $\norm{Y_r}_2 = \sqrt{N + 1}$, as we are considering labels in $\{-1, 1\}$.

\end{itemize}
Thus, we conclude that
\begin{equation}
    \left| f_{\textup{RF}}(X^i(\Delta), \theta^*_{r}) - f_{\textup{RF}}(X, \theta^*_{r}) \right| = \bigO{\sqrt{\frac{N}{n}}},
\end{equation}
with high probability. As a consequence, in the regime where $N = o(n)$, $f_{\textup{RF}}(\cdot, \theta^*_{r})$ cannot distinguish between the samples $X$ and $X^i(\Delta)$, without the additional need for Assumption \ref{ass:uncertainity}.
This result is also shown in our experiments, in the left subplots of Figure \ref{fig:gen}, where the points approach smaller values of $\gamma$ and consequently higher values of the error as $N$ decreases, becoming comparable with $n$. 


\section{Further experiments}\label{app:exp}




In Figure \ref{fig:gen}, we compute $\Delta^*$ by optimizing with respect to it
the following two losses (for fine-tuning and re-training, respectively):
\begin{equation}\label{eq:orloss1}
    \ell_{\theta^*_f}(\Delta) := \left( \frac{\varphi_{\textup{RAF}}(X^i(\Delta))^\top \varphi_{\textup{RAF}}(X)}{\norm{\varphi_{\textup{RAF}}(X)}_2^2} + 1 \right)^2,
\end{equation}
\begin{equation}\label{eq:orloss2}
    \ell_{\theta^*_r}(\Delta) := \left( \mathcal F_{{\textup{RAF}}}(X, X^i(\Delta)) + 1 \right)^2,
\end{equation}
subject to the constraint $\norm{\Delta^*} \leq \sqrt d$.
We also report with cross markers the points obtained optimizing the loss $\ell_{\textup{Err}}(\Delta) := \textup{Err}_{\textup{RAF}}(X^i(\Delta), \theta^*_{f/r})$, showing that this method provides less interesting results, as the error tends to be minimized at the expenses of a large value of $\gamma$.

An alternative approach is to still optimize with respect to the errors directly, but after introducing a penalty term $p$ on the value of $\gamma$. In particular, we consider the following penalized loss
\begin{equation}\label{eq:penalized_loss}
    \ell_{\textup{Err}, p}(\Delta) := \textup{Err}_{\textup{RAF}}(X^i(\Delta), \theta^*_{f/r}) + p \left( f_{\textup{RAF}}(X^i(\Delta), \theta^*) - f_{\textup{RAF}}(X, \theta^*) \right)^2.
\end{equation}
We perform new experiments with this optimization algorithm, for different values of $p$, and we report the results in Figure \ref{fig:app2}. In this case, compared to the loss $\ell_{\textup{Err}}(\Delta)$, we can more easily obtain points that lie below the lower bound. However, it remains difficult to find points where both the error and $\gamma$ are small. 


Another optimization option is to introduce a penalty term to the losses in \eqref{eq:orloss1} and \eqref{eq:orloss2}:
\begin{equation}\label{eq:penalized_f}
    \ell_{\theta^*_f, p}(\Delta) := \left( \frac{\varphi_{\textup{RAF}}(X^i(\Delta))^\top \varphi_{\textup{RAF}}(X)}{\norm{\varphi_{\textup{RAF}}(X)}_2^2} + 1 \right)^2  + p \left( f_{\textup{RAF}}(X^i(\Delta), \theta^*) - f_{\textup{RAF}}(X, \theta^*) \right)^2
\end{equation}
and
\begin{equation}\label{eq:penalized_r}
    \ell_{\theta^*_r, p}(\Delta) := \left( \mathcal F_{{\textup{RAF}}}(X, X^i(\Delta)) + 1 \right)^2 + p \left( f_{\textup{RAF}}(X^i(\Delta), \theta^*) - f_{\textup{RAF}}(X, \theta^*) \right)^2,
\end{equation}
for the fine-tuning and re-training case respectively. We perform new experiments with this optimization algorithm, for different values of $p$, and we report the results in Figure \ref{fig:app3}. In this case, it is easier to obtain final points that respect Assumption \ref{ass:uncertainity} with a lower value of $\gamma$.


\begin{figure*}
    \centering
    \includegraphics[width=\textwidth]{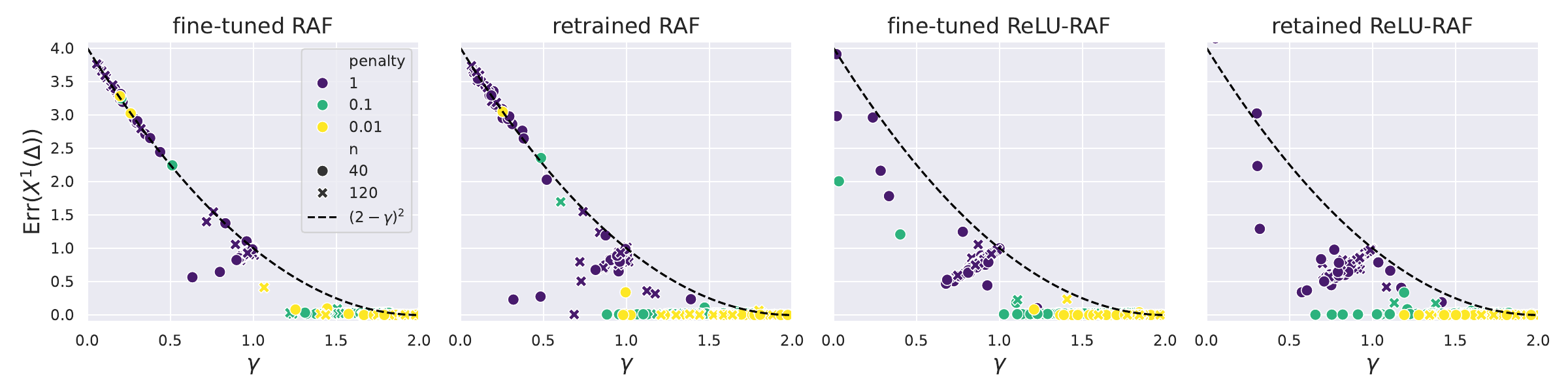}
    
    \caption{$\textup{Err}_{\varphi}(X^i(\Delta), \theta^*_{f/r})$ for the RAF (two left sub-plots) and ReLU-RAF (two right subplots) maps, as a function of the smallest $\gamma$ for which Assumption \ref{ass:uncertainity} is satisfied. Every sub-plot has a fixed context length $n = \{40, 120\}$, embedding dimension $d = 768$ and number of training samples $N =  400$. Every point in the scatter-plots represents an independent simulation where $(X, y)$ and $(\mathcal X, \mathcal Y)$ are the BERT-Base embeddings of a random subset of the imdb dataset (after pre-processing to fulfill Assumption \ref{ass:d}).
    For every point, $\Delta$ is obtained through constrained gradient descent optimization of $\ell_{\textup{Err}, p}(\Delta)$, defined in \eqref{eq:penalized_loss}, for different values of the penalty term $p = \{1, 0.1, 0.01\}$.}
    \label{fig:app2}
\end{figure*}

\begin{figure*}[!ht]
    \centering
    \includegraphics[width=\textwidth]{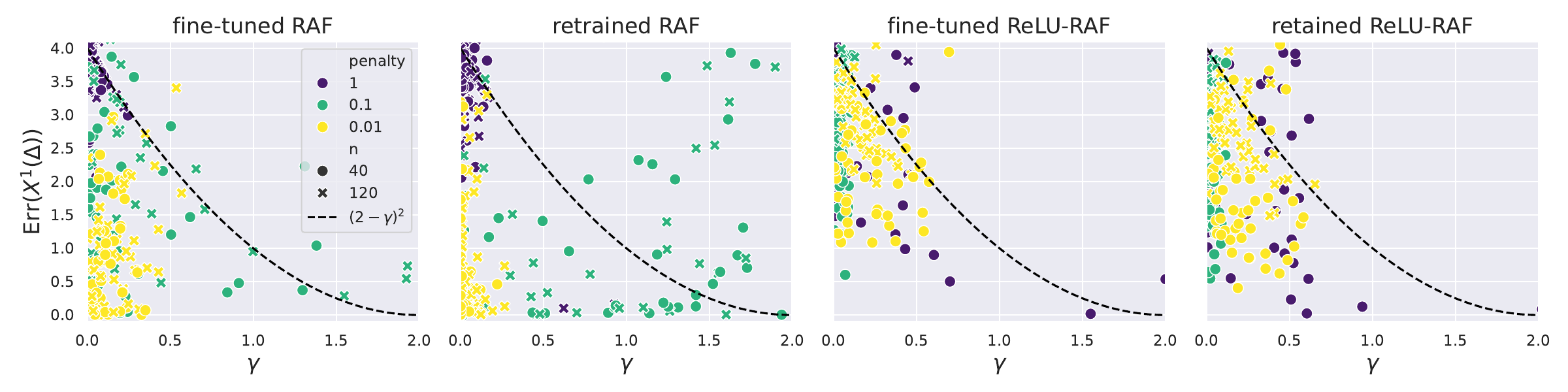}
\caption{$\textup{Err}_{\varphi}(X^i(\Delta), \theta^*_{f/r})$ for the RAF (two left sub-plots) and ReLU-RAF (two right subplots) maps, as a function of the smallest $\gamma$ for which Assumption \ref{ass:uncertainity} is satisfied. For every point, $\Delta$ is obtained through constrained gradient descent optimization of $\ell_{\theta^*_r, f}(\Delta)$ in the fine-tuned case, and $\ell_{\theta^*_r, p}(\Delta)$ in the re-trained case, defined in \eqref{eq:penalized_f} and \eqref{eq:penalized_r} respectively, for different values of the penalty term $p = \{1, 0.1, 0.01\}$. The rest of the setup is equivalent to the one described in Figure \ref{fig:app2}.}
    \label{fig:app3}
\end{figure*}

Finally, we consider swapping the losses $\ell_{\theta^*_f}$ and $\ell_{\theta^*_r}$ defined above, i.e., employ the former for re-training and the latter for fine-tuning. Given the heuristic nature of these losses, it is a priori not obvious that they perform their best on the respective setting, as they could be interchangeable. In Figure \ref{fig:app4}, we report the results of this investigation. We report in deep-blue the points resulting from the optimization of $\ell_{\theta^*_f}$, and in yellow the points resulting from the optimization of $\ell_{\theta^*_r}$, for both the fine-tuned and re-trained setting. We note that $\ell_{\theta^*_f}$ performs better on the fine-tuned solution, and $\ell_{\theta^*_r}$ better on the re-trained one.

\clearpage

\begin{figure*}
\vspace{-35em}
    \centering
    \includegraphics[width=\textwidth]{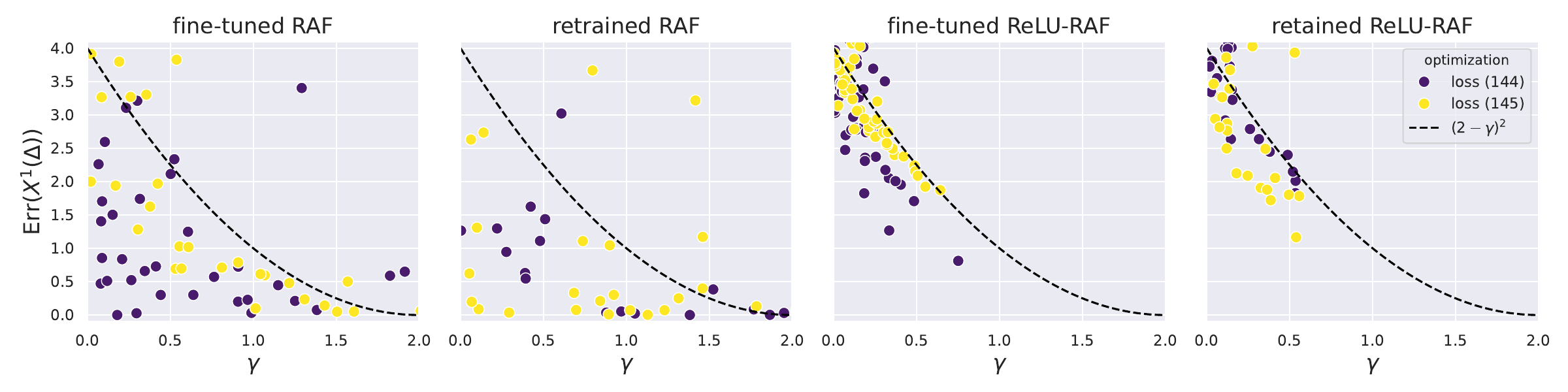}
\caption{$\textup{Err}_{\varphi}(X^i(\Delta), \theta^*_{f/r})$ for the RAF (two left sub-plots) and ReLU-RAF (two right subplots) maps, as a function of the smallest $\gamma$ for which Assumption \ref{ass:uncertainity} is satisfied. We consider fixed embedding dimension $d = 768$, context length $n = 120$, and number of training samples $N =  400$. Every point in the scatter-plots represents an independent simulation where $(X, y)$ and $(\mathcal X, \mathcal Y)$ are the BERT-Base embeddings of a random subset of the imdb dataset (after pre-processing to fulfill Assumption \ref{ass:d}).
    For every point, $\Delta$ is obtained through constrained gradient descent optimization of either  $\ell_{\theta^*_f}$, defined in \eqref{eq:orloss1}, or $\ell_{\theta^*_r}$, defined in \eqref{eq:orloss2}.}
    \label{fig:app4}
\end{figure*}

\end{document}